\documentclass[review,hidelinks,onefignum,onetabnum]{siamart220329}

\usepackage[numbers,compress,sort]{natbib}

\usepackage{lipsum}
\usepackage{amsfonts}
\usepackage{graphicx}
\usepackage{epstopdf}
\usepackage{algorithmic}
\usepackage{amssymb}

\usepackage{enumerate}

\usepackage{float}  
\usepackage{subfigure}
\usepackage{marvosym}

\ifpdf
  \DeclareGraphicsExtensions{.eps,.pdf,.png,.jpg}
\else
  \DeclareGraphicsExtensions{.eps}
\fi

\newsiamremark{remark}{Remark}
\newsiamremark{hypothesis}{Hypothesis}
\crefname{hypothesis}{Hypothesis}{Hypotheses}
\newsiamthm{claim}{Claim}
\newsiamthm{notation}{Notation}
\newsiamthm{example}{Example}

\usepackage{hyperref}
\usepackage{amssymb}
\usepackage{amsmath}
\usepackage{subcaption}
\usepackage{cleveref}

\headers{A View from Polytope Decomposition}{Li, Zhang, Zeng, Fan}

\title{Neural Network Approximation: A View from Polytope Decomposition\thanks{Submitted to the editors DATE.
\funding{This work is supported by the Direct Grant for Research from the Chinese University of Hong Kong and ITS/173/22FP from the Innovation and Technology Fund of Hong Kong}}}

\author{
Ze-Yu Li\thanks{Department of Mathematics, The Chinese University of Hong Kong, Hong Kong 
  (\email{zyli@math.cuhk.edu.hk}).} 
  \and Shijun Zhang\thanks{Department of Applied Mathematics, The Hong Kong Polytechnic University, Hong Kong (\email{shijun.zhang@polyu.edu.hk}).}
\and Tieyong Zeng\textsuperscript{\Letter}  \thanks{Department of Mathematics, The Chinese University of Hong Kong, Hong Kong 
  (\email{zeng@math.cuhk.edu.hk}).}
\and
  Feng-Lei Fan\thanks{Department of Mathematics, The Chinese University of Hong Kong, Hong Kong 
  (\email{hitfanfenglei@gmail.com}).}  
  }

\usepackage{amsopn}

\ifpdf
\hypersetup{
  pdftitle={Universal Approximation by Neural Networks Based on Polytope Decomposition},
  pdfauthor={TBD}
}
\fi

\usepackage{color}

\usepackage{amssymb}
\usepackage{amsmath}

\usepackage{bm}
\usepackage{xcolor}

\DeclareFontFamily{U}{mathx}{}
\DeclareFontShape{U}{mathx}{m}{n}{<-> mathx10}{}
\DeclareSymbolFont{mathx}{U}{mathx}{m}{n}
\DeclareMathAccent{\widehat}{0}{mathx}{"70}
\DeclareMathAccent{\widecheck}{0}{mathx}{"71}

\let\tilde\widetilde
\let\hat\widehat
\let\epsilon\varepsilon

\let\subset\subseteq

\let\cdots\customcdots
\let\myforall\forall
\def\forall{{\myforall\, }}
\let\myexists\exists
\def\exists{{\myexists\, }}

\usepackage{tabularx,multirow,array,booktabs}
\usepackage{colortbl}
\usepackage{mathtools}
\usepackage{hyperref}

\begin{document}
\nolinenumbers
\makeatletter
\renewcommand{\footercopyright}{}
\makeatother

\maketitle

\begin{abstract}
 {
Universal approximation theory offers a foundational framework to verify neural network expressiveness, enabling principled utilization in real-world applications.
However, most existing theoretical constructions are established by uniformly dividing the input space into tiny hypercubes without considering the local regularity of the target function. In this work, we investigate the universal approximation capabilities of ReLU networks from a view of polytope decomposition, which offers a more realistic and task-oriented approach compared to current methods. To achieve this, we develop an explicit kernel polynomial method to derive an universal approximation of continuous functions, which is characterized not only by the refined Totik-Ditzian-type modulus of continuity, but also by polytopical domain decomposition. Then, a ReLU network is constructed to approximate the kernel polynomial in each subdomain separately. Furthermore, we find that polytope decomposition makes our approximation more efficient and flexible than existing methods in many cases, especially near singular points of the objective function. Lastly, we extend our approach to analytic functions to reach a higher approximation rate.}

\end{abstract}








\section{Introduction}

\begin{figure}[htb]
    \centering    \includegraphics[width=0.6\linewidth]{ 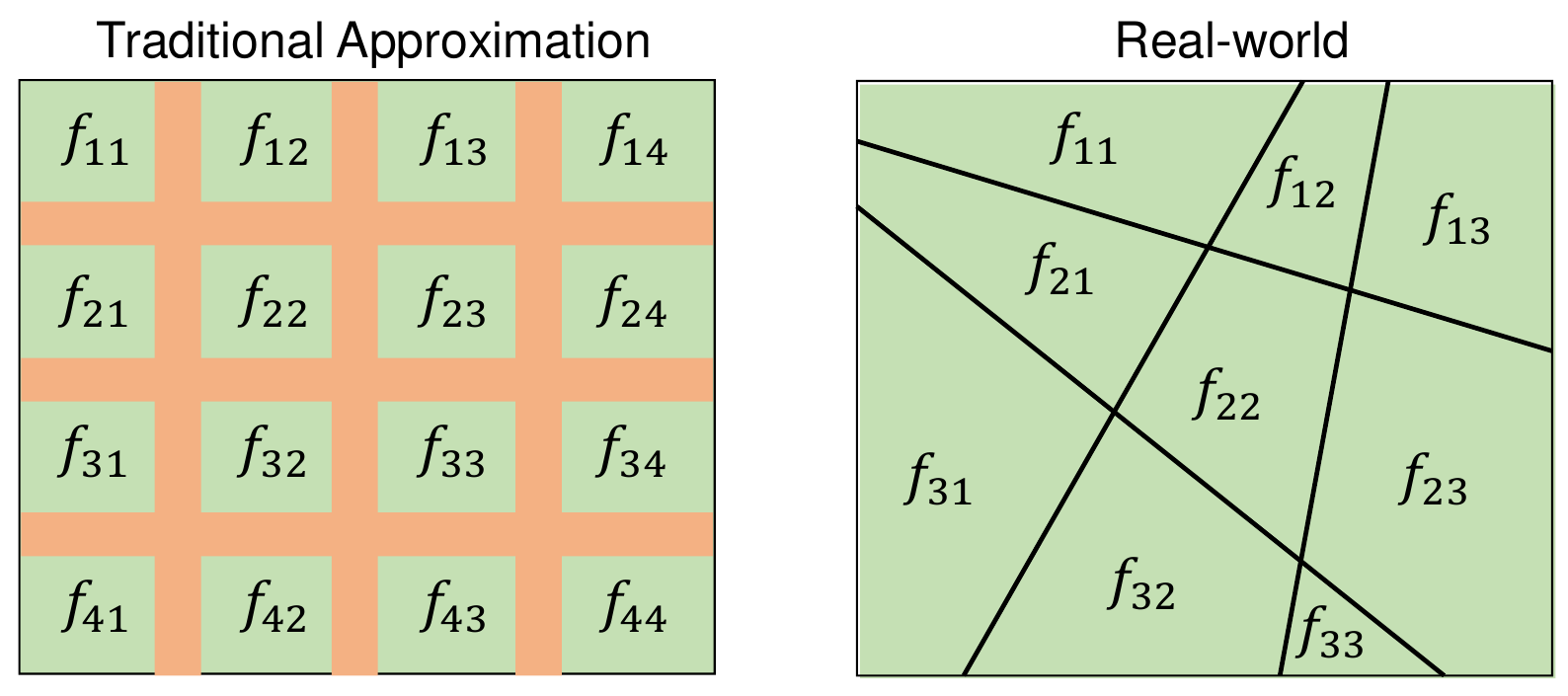}
    \caption{While universal approximation theories partition the space into hypercubes, a ReLU network actually divides the space into polytopes.}
    \label{fig:contrast}
    \vspace{-0.3cm}
\end{figure}

Over the past few years, deep networks have driven remarkable progress in a wide array of fields such as computer vision, natural language processing, scientific computing, and robotics. The fundamental power of deep networks stems from their ability to represent highly complex information by composing multiple layers of affine transformations interleaved with non-linear activation functions. This layered structure allows them to capture intricate patterns and dependencies, making them incredibly versatile for handling diverse, high-dimensional data.
A substantial body of research has explored the expressive capacity of deep neural networks, aiming to understand their potential to approximate complex functions. However, the theoretical frameworks underpinning this expressive power are often regarded as idealistic; the constructions they propose, while mathematically sound, are usually not practically learnable by neural networks. Two notable issues leading to this gap are the way networks partition the input space and the task-agnostic nature of these theoretical constructions, which lack specificity to particular learning contexts. 
\begin{itemize}
    \item 
 As Figure~\ref{fig:contrast} shows, the existing approximation theories of ReLU networks usually decompose the input space into hypercubes, and the approximation error is minimized by gradually adopting hypercubes of finer resolutions \cite{yarotsky2017error,shen2019deep,shen2022optimal}. Although some literature  extended the irregular domain to a larger hypercube \cite{shen2019deep} or mapped a low-dimensional manifold to a hypercube \cite{Chen2019NonparametricRO}, they were still essentially based on hypercube decomposition.  However, a ReLU network in practice partitions the space into polytopes, and the approximation error is also reduced by adjusting the shapes of polytopes. 
 \item In the majority of the existing approximation theories, the whole domain is uniformly divided, without considering the local variations in continuity and smoothness properties \cite{Herrmann2022ConstructiveDR,lu2021deep,liang2016deep}. In reality, this is not what a network does for functions lacking continuity or smoothness on some part of the domain. If the regularity condition (smoothness, continuity, and analytical properties) varies on each part of the domain, a network usually exhibits adaptivity on different parts.
\end{itemize}

These limitations highlight the need for more practical and nuanced theories that can fill the gap between theoretical expressive power and real-world learnability in deep networks. Here we develop a   {novel} approximation theory based on polytope decomposition. Specifically, our theoretical construction divides the input domain into a union of convex polytopes, and can adapt to the local regularity of the target function because polytope decomposition is more flexible. The key difficulty our construction addresses is that the shape of a polytope is fixed once the partition is given. Thus, we need to use a small number of subsets (parallelepipeds or simplexes) to cover the polytope, which shall fulfill two properties: i) no region outside the polytope is covered to avoid interference between neighboring polytopes; ii) almost all region (except for a   trifling region arbitrarily small) inside the polytope is covered to ensure a complete approximation to the target function. Then, we need to approximate the target function on such subset coverings. Although polynomial approximation over convex polytopes has been investigated \cite{Totik2020,Totik2014PolynomialAO,ditzian2012moduli} in polynomial approximation theory, they are not suitable for a direct use in network approximation. On the one hand, the constructions are usually not completely explicit, thus causing difficulty in bound estimation. On the other hand, some operations are easy to perform by neural networks but hard by polynomials. Thus, the efficiency will go down if we use such results directly. Our innovation is to use the kernel polynomial method to disentangle the coefficients and the degree of a polynomial, thereby controlling the magnitude of coefficients and leading to an explicit error estimation. Next, when one can approximate a function over a polytope, it facilitates a much more flexible tool such that one can partition the space more precisely according to the regularity of the targeted function, which means the task adaptivity is automatically gained in our construction.

\subsection{Main results}

 {
\begin{theorem}\label{thm:MainThmContinuous}
	Let $K\subseteq\mathbb{R}^d$ be a polytope, and $f: K\to\mathbb{R}$ be continuous. Then for any $N\in \mathbb{}^{+},\alpha>0$, there exists a  function $\tilde{f}:\mathbb{R}^d\to \mathbb{R}$ vanishing outside $K$ that can be implemented by a ReLU network of width $\mathcal{O}\left(\max\left
    \{N^3,N^{d-1}\log{N}\right\}\right)$ and depth $\mathcal{O}(\log N)$ such that
	\begin{equation}
		\left\|\tilde{f}-f\right\|_{L^{\infty}(K^{\prime})}\leq \mathcal{O}\left(\tilde{\omega}_K\left(f,N^{-1}\right)+N^{-\alpha}\right),
	\end{equation}
where $K^{\prime}\subseteq K$ with $\mathfrak{m}(K\setminus K^{\prime})$ is arbitrary small, and $\mathfrak{m}$ is the Lebesgue measure.
\end{theorem}
}
 {
Note that the main part of the approximation error is given by the Totik-Ditzian modulus of continuity $\tilde{\omega}(f,\cdot)$ given by
\begin{equation*}
\label{eq_TDmodulus}
    \tilde{\omega}_{K}(f, t)=\sup _{\substack{\boldsymbol{x} \in K, \boldsymbol{e} \in \mathcal{E}^{K} \\ 0<h \leq t}}\left| f\left(\boldsymbol{x}+\frac{h \tilde{d}_{K}(\boldsymbol{e}, \boldsymbol{x})}{2}\boldsymbol{e}\right)-f\left(\boldsymbol{x}-\frac{h \tilde{d}_{K}(\boldsymbol{e}, \boldsymbol{x})}{2}\boldsymbol{e}\right)\right|,
\end{equation*}
where the normalized distance $\tilde{d}_K(\boldsymbol{e},\boldsymbol{x})$ is defined as the geometric mean of the distances from the two endpoints of $K$ to $\boldsymbol{x}$, intersected by a line passing through $\boldsymbol{x}$ and parallel to $e$, \textit{i.e.}, 
\begin{equation*}
    \tilde{d}_K(\boldsymbol{e},\boldsymbol{x})=\sqrt{d(\boldsymbol{x},\boldsymbol{a}_{\boldsymbol{e},\boldsymbol{x}}) \cdot d(\boldsymbol{x},\boldsymbol{b}_{\boldsymbol{e},\boldsymbol{x}})},
\end{equation*}
 with $\boldsymbol{a}_{\boldsymbol{e},\boldsymbol{x}}$ and $\boldsymbol{b}_{\boldsymbol{e},\boldsymbol{x}}$ being the endpoints of the intersection segment between $K$ and the line through $\boldsymbol{x}$ parallel to $\boldsymbol{e}$.
}

 {
Totik-Ditzian modulus of continuity $\tilde{\omega}_K(f,\cdot)$ is essentially different from the ordinary modulus of continuity (see \cref{Figure_TotikModulus}), which is given by
\begin{equation*}
    \omega_K(f,  t)=\sup
  \Big\{ | f(\boldsymbol{x})-f(\boldsymbol{y})|:  \boldsymbol{x},\boldsymbol{y}\in K, \; \left\|\boldsymbol{x}-\boldsymbol{y}\right\|_2\leq t\Big\}.
\end{equation*}
}
 {
The former weighs the difference step according to the distance from the boundary when calculating the modulus of continuity, which is widely used in modern approximation theory. In comparison, the latter fails to characterize the approximation behavior of continuous functions near the boundary (see \cite{nikolskii1946,ditzian2007polynomial}). 
}
\begin{figure}[htb!]
\center{\includegraphics[width=\linewidth] {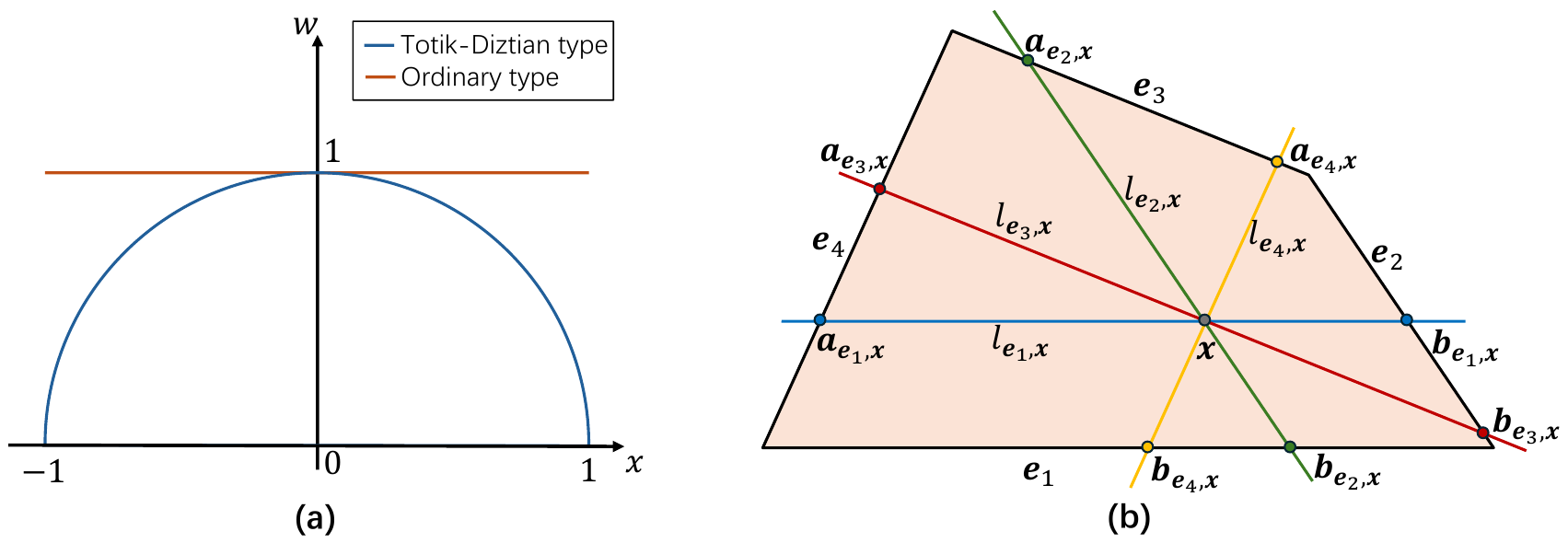}}
\caption{Illustration of the Ditzian-Totik modulus of continuity. (a) The modulus of continuity of $f:[-1,1]\to \mathbb{R}$ with step $t$ is computed by taking supremum of $f(x+\frac{hw(x)}{2})-f(x-\frac{hw(x)}{2})$ over $0<h<t$ and $-1\leq x\leq 1$, where $w(x)$ is a weight function measuring the distance from $x$ to the boundary. When $w(x)\equiv1$ and $w(x)=\phi(x)=\sqrt{1-x^2}$, we obtain the ordinary and Ditzian-Totik modulus of continuity respectively. (b) $\boldsymbol{e}_{i},i=1,2,3,4$ are the direction of edges of polytope $K$. For $\boldsymbol{x}\in K$, $l_{\boldsymbol{e_i},\boldsymbol{x}}$ is a straight line through $\boldsymbol{x}$ and parallel to $\boldsymbol{e}_i$ intersecting $K$ at $\boldsymbol{a}_{\boldsymbol{e}_i,\boldsymbol{x}}$ and $\boldsymbol{b}_{\boldsymbol{e}_i,\boldsymbol{x}}$.}
\label{Figure_TotikModulus}
\end{figure}

 {
Our approach can also be extended to approximating analytic functions, where the approximation rate is exponential in width and depth.
\begin{theorem}\label{thm:MainThmAnalytic}
	Let $K\subseteq [-1,1]^d$ be a polytope. Suppose that $f: K\to\mathbb{R}$ can be analytically continued to the region
\begin{equation*}
    \left\{x\in\mathbb{C}^d:\left\|x\right\|+\left|\left\|x\right\|-d\right|\leq d+2h^2\right\}.
\end{equation*}
    Then for any $N\in \mathbb{Z}^{+}$, there exists a  function $\tilde{f}:\mathbb{R}^d\to \mathbb{R}$ vanishing outside $K$ that can be implemented by a ReLU network with of $\mathcal{O}\left(\max\left
    \{N^3,N^d\right\}\right)$ and depth $\mathcal{O}(N)$ such that
	\begin{equation}
		\left\|\tilde{f}-f\right\|_{L^{\infty}(K^{\prime})}\leq \mathcal{O}_{\epsilon}\left((h+\sqrt{1+h^2})^{-N/\sqrt{d}}\right),
	\end{equation}
where $K^{\prime}\subseteq K$ with $\mathfrak{m}(K\setminus K^{\prime})$ is arbitrary small, and $\rho=h+\sqrt{1+h^2}$. Here $g(n)=\mathcal{O}_{\epsilon}\left(a^{-n}\right)$ if for all $\epsilon>0$, $g(n)=\mathcal{O}\left((a-\epsilon)^{-n}\right)$.
\end{theorem}
}
 {
 Please note that our approximator $\bar{f}:\mathbb{R}^d\to\mathbb{R}$ given by the ReLU network vanishes outside $K$, \textit{i.e.},
\begin{equation*}
    \text{supp}\left(\bar{f}\right)=\left\{\boldsymbol{x}\in\mathbb{R}^d:\bar{f}(\boldsymbol{x})\neq0\right\}\subseteq K.
\end{equation*}
Unlike existing approximation results (\cite{shen2019deep,lu2021deep,yarotsky2017error}), this allows us to approximate the target function on each polytope subdomain separately, based on the local regularity of the target function. For example, a function may satisfy only basic continuity conditions globally, but is highly smooth or analytic locally. This situation often occurs when the data distribution is different, and the representation of the neural network is not smooth near the decision boundary.
}

 {
In summary, our contributions are threefold:
}
\begin{enumerate}[i)]
    \item \textbf{  {More} Realistic universal approximation}: Most neural network approximation theories are based on cube decomposition. In line with the real-world ReLU networks' learning fashion, we give a more realistic universal approximation theorem over an irregular domain based on polytope decomposition. Accordingly, the approximation is characterized by the Totik-Ditzian-type modulus of continuity, which depends not only on the property of the target function, but also on the polytope decomposition. To the best of our knowledge, this is the first result that considers polytope decomposition in establishing network approximation theory.

    \item \textbf{Task-aware network architecture}: Network approximation for various function classes, \textit{e.g.}, continuous functions \cite{lu2021deep,shen2022optimal}, smooth functions \cite{shen2019deep,yarotsky2017error}, and analytic functions \cite{Herrmann2022ConstructiveDR}, has been well studied. However, Our construction can be highly adaptive to the target function, \textit{i.e.}, the network topology varies depending on the local regularity, which can improve approximation efficiency and optimize existing approximation method.

    \item \textbf{Introducing KPM into network approximation:} To the best of our knowledge, our work is the first to introduce the kernel polynomial method (KPM, \cite{Weisse2005TheKP,Bjeli2022ChebyshevKP}) into network approximation theory, which is beyond the paradigm of Taylor expansion existing in most existing works. Due to the importance of the KPM method in approximation theory, we extend our main result to analytic functions and higher order modulus of smoothness, which further opens the door to a wide variety of new research opportunities.
    
\end{enumerate}

\subsection{Superiority of polytope decomposition}

 {
On the one hand, the Totik-Ditzian modulus of continuity better characterizes the variation of the target function and aligns with the polytope decomposition. The comparison in the following manifests the essential difference.
}

 {
In \cite{shen2019deep}, the authors showed that a ReLU network of width $\mathcal{O}(N)$ and depth $\mathcal{O}(L)$ can reach an approximation rate $\mathcal{O}\left(\omega\left(f,N^{-2/d}L^{-2/d}\right)\right)$ for continuous functions on $[0,1]^d$, \textit{i.e.}, an error of $\mathcal{O}\left(\omega\left(f,N^{-2+2/d}(\log{N})^{-4/d}\right)\right)$ for the network of the size in \cref{thm:MainThmContinuous} when $d\geq 4$. Although this result is nearly optimal for Holder continuous functions on $[0,1]^d$, it generally inaccurate when the function exhibits a singularity near the boundary.
}

 {
\begin{example}
\label{example}
    Consider on $[0,1]$
\begin{equation*}
    f(x)=\begin{cases}
        -x^{1/3}\ln{x},&x\in (0,1],\\
        0,&x=0.
    \end{cases}
\end{equation*}
    Then we have $\omega_{[0,1]}(f,t)\geq \mathcal{O}(t^{1/3}|\ln{t}|)$ but $\tilde{\omega}_{[0,1]}(f,t)=\mathcal{O}(t^{2/3}|\ln{t}|)$.
\end{example}
}

\begin{proof}
 {
\begin{equation*}
    f^{\prime}(x)=-x^{-2/3}-\frac{1}{3}x^{-2/3}\ln{x}\sim x^{-2/3}\left| \ln{x} \right|,\quad \text{as }x\to 0,
\end{equation*}
which is decreasing on $(0,1)$.
}

 {
For the ordinary modulus of continuity, we have
\begin{equation*}
\begin{split} \omega_{[0,1]}\left(f,t\right)
&=\sup_{0<h\leq t,x\in [0,1-h]} \left| f(x+h)-f(x) \right|\\
&=\sup_{0<h\leq t,x\in [0,1-h]} \left| f^{\prime}(\xi_{x})h \right|,\quad x\leq\xi_{x}\leq x+h,\\
&\geq \sup_{0<h\leq t,x\in [0,1-h]} \left| f^{\prime}(x+h)h \right|\\
&= \sup_{0<h\leq t} \left| f^{\prime}(h)h \right|\sim t^{1/3}\left| \ln{t} \right|,\quad \text{as }t\to 0.
\end{split}   
\end{equation*}
}

 {
To compute the Ditzian-Totik modulus of continuity, we first use a result given by Theorem 3.3.2 and Remark 3.3.3 in \cite{ditzian2012moduli}, which gives for all small $t>0$
\begin{equation*}
    \tilde{\omega}_{[0,1]}(f,t)=\sup_{0<h\leq t,x\in \left[\frac{h^2}{5},1-\frac{h^2}{5}\right]}\left| \Delta_{h\sqrt{x(1-x)}} f(x)\right|.
\end{equation*}
Following a similar procedure to the ordinary case, there is some $\eta_{x},x-\frac{h}{2}\sqrt{x(1-x)}\leq \eta_{x}\leq x+\frac{h}{2}\sqrt{x(1-x)}$ such that
\begin{equation*}
    \left|\Delta_{h\sqrt{x(1-x)}} f(x)\right|=\left|f^{\prime}\left(\eta_{x}\right)\right|h\sqrt{x(1-x)}
\end{equation*}
which further gives
\begin{equation*}
\begin{split}
    \tilde{\omega}_{[0,1]}(f,t)&\sim \sup_{0<h\leq t,x\in \left[\frac{h^2}{5},1-\frac{h^2}{5}\right]} \eta_x^{-2/3}|\ln{\eta_x}|h\sqrt{x(1-x)}.\\
    & \sim \sup_{0<h\leq t,x\in \left[\frac{h^2}{5},1-\frac{h^2}{5}\right]} x^{-2/3}|\ln{x}|h\sqrt{x(1-x)}\\
    &\sim t^{-4/3}|\ln{t}|t\sqrt{t^2}=t^{2/3}|\ln{t}|.
\end{split}
\end{equation*}
A more detailed calculus and a more general case can be found in Chapter 3.4 of \cite{ditzian2012moduli}.
}
\end{proof}

 {
Furthermore, let $F(x)=\sum_{j=1}^d f(x_j),d\geq 4$, then a network of depth $\mathcal{O}(\log N)$ and width $\mathcal{O}(N^{d-1}\log{N})$ can approximate $F(x)$ with the error of $\mathcal{O}(N^{-2/3}|\ln N|)$, which is significantly smaller than $\mathcal{O}\left(N^{-1/2}|\log N|^{-1/3}|\ln{\ln{N}}|\right)$ given by \cite{shen2019deep}.
}

 {
On the other hand, polytope decomposition can better utilize the local regularity. For example, considering the continuous but nonanalytic target function $\max\left\{e^{x-y},e^{y-x}\right\}$ on $[0,1]^2$. Doing approximation on two simplices $\left\{(x,y)\in[0,1]^2:x\geq y\right\}$ and $\left\{(x,y)\in[0,1]^2:x\leq y\right\}$ separately can lead to a much faster approximate rate. What is worse is that when the objective function is discontinuous over the entire domain, traditional approximation schemes are incompetent.
}

 {
Lastly, we emphasize that in our construction, the choice of polytope decomposition also affects the approximation efficiency, which aligns with the real situation.
\begin{example}
\label{example_2}
    Consider on $[-1,1]^2$,
\begin{equation*}
    f(x,y)=\begin{cases}
        \psi(x,y)\sqrt{x^2-y^2}\log\frac{1}{x^2+y^2},&x,y\in [-1,0)\cup(0,1],\\
        0,&x=0\text{ or }y=0.
    \end{cases}
\end{equation*}
We partition the square $[-1,1]^2$ into four closed triangular regions $S_i,i=1,2,3,4,$ by its diagonals, \textit{i.e.},
\begin{equation*}
    S_i=\left\{ (x,y)=(r\cos\theta,r\sin\theta): r\in [0,\sec\left(\theta-(i-1)\pi/2\right)],\theta\in [(i-1)\pi/2,i\pi/2] \right\}.
\end{equation*}
Then $\tilde{\omega}_{S_i}(f,t)=\mathcal{O}(t),i=1,2,3,4$ while $\tilde{\omega}_{[-1,1]^2}(f,t)=\mathcal{O}(t|\log{t}|)$.
\end{example}
}

 {
\begin{remark}
    The above example comes from Example 6.4 in \cite{Ditzian1984ModuliOC} and Remarks 12.2.1 (3) in \cite{ditzian2012moduli}, where they construct $f(x,y)=\psi(x,y)\sqrt{xy}\log\frac{1}{x^2+y^2}$ such that the modulus of continuity along $(0,1)$ and $(1,0)$ is $\mathcal{O}(t)$ while the modulus of continuity along $(1,1)/\sqrt{2}$ is $\mathcal{O}(t|\ln{t}|)$. We rotate this function by $\pi/4$ to obtain \cref{example_2}. \textbf{In this example, if we decompose the domain $[-1,1]^2$ into $S_i,i=1,\ldots,4$ and make an approximation on each subdomain separately, we can improve the approximate rate from $\mathcal{O}(N^{-1}\ln{N})$ to $\mathcal{O}(N^{-1})$.}
\end{remark}
}

\subsection{Related Work}

We review work from two directions: approximation theory over polytopes and theory on approximation capabilities of neural networks, which are most closely related to ours.

\textbf{Approximation Theory over Polytopes.} The polynomial approximation over polytopes is a heating topic in approximation theory, which aims to find polynomials whose approximation error is bounded by the modulus of smoothness. \cite{DITZIAN1989Best,Berens1991KmoduliMO} studies the best polynomial approximation and K-functional over a simplex. \cite{Ditzian1996Polynomial,ditzian2012moduli,CHEN1991Best} studies polynomial approximation to $L^p$ functions over simple polytopes. Due to the absence of K-functional on general polytopes, the approximation over general polytopes remained an unsolved open problem for decades before Totik's work \cite{Totik2014PolynomialAO} for continuous and $L^p$ functions in $\mathbb{R}^3$ and \cite{Totik2020} for continuous functions in $\mathbb{R}^d$. 

\textbf{Approximation of Neural Networks.} 
In the early stage of this field, attention was focused on universal approximation theory. For example, \cite{lu2017expressive,hanin2019universal} demonstrated that ReLU networks, with bounded width and unbounded depth, can serve as universal approximators. \cite{fan2021sparse} showed that a ReLU network with a single neuron width, coupled with shortcuts connecting all hidden neurons to the output neuron, can approximate any univariate function. \cite{lin2018resnet} proved that a ResNet with one-neuron-wide hidden layers is a universal approximator.
Given the widespread use of ReLU, recent studies 
have focused on characterizing the approximation capacity of ReLU networks. As shown in Table~\ref{tab:universal_approximation}, these studies illustrate that ReLU networks can effectively approximate various types of functions, including continuous functions \cite{lu2021deep,shen2022optimal}, smooth functions \cite{shen2019deep,yarotsky2017error}, and analytic functions \cite{Herrmann2022ConstructiveDR}. The approximation rates vary according to network properties such as width, depth, and parameter count. For instance, \cite{yarotsky2017error} proved that deep ReLU networks can uniformly approximate functions in Sobolev space. \cite{chen2019efficient} demonstrated that ReLU networks can efficiently approximate functions over low-dimensional manifolds. \cite{guhring2020error} examined the efficiency of deep ReLU networks in approximating Sobolev regular functions using weaker Sobolev norms. \cite{lu2021deep} established the optimal approximation error for deep ReLU networks when approximating smooth functions, where width and depth are log-linearly scaled. \cite{shen2022deep} showed that a ReLU network can approximate any Lipschitz continuous function with a small number of learnable parameters. \cite{daubechies2022neural} proved that rough but refinable functions can be approximated by deep ReLU networks with a fixed width. Additionally, \cite{shen2021neural} employed special activation functions to construct three-layer networks capable of approximating any continuous function with an error exponentially decaying with width.

\begin{table}[htbp!]
\def\arraystretch{1.15}
 \centering
\caption{
Comparisons between our work and the existing neural network approximation results.
Here, $\omega(f,\cdot)$ is the modulus of continuity of the function $f$, $\tilde{\omega}(f,\cdot)$ is the Ditzian-Totik modulus of smoothness defined in \eqref{eq_TDmodulus}, $K$ is a polytope, and $X$ is a compact set. Our \cref{thm:MainThmAnalytic} achieves the superior approximation rate, where both width and depth are in the exponent.
}	
	\resizebox{0.98\linewidth}{!}{ 
		\begin{tabular}{ ccccccccccc  }
      \toprule
      Reference & Partition & Target Function & Width & Depth & Error \\
      \midrule
    \cite{cybenko1989approximation} & cube & $f \in C([0,1]^d)$ & - & - \\
      \midrule
       \cite{shen2021neural} & cube &  $f \in C([0,1]^d)$& $\max\{d,N\}$ & $3$ &$ \mathcal{O}\left(2 w_f(2\sqrt{d})2^{-N}+w_f(2\sqrt{d}2^{-N})\right)$  \\
   \midrule
   \cite{lu2021deep}& cube & $f \in C^s([0,1]^d)$&$\mathcal{O}(N\ln{N})$&$\mathcal{O}(L\ln{L})$&$ \mathcal{O}\left(\Vert f \Vert_{C^s([0,1]^d)} N^{-2s/d}L^{-2s/d}\right)$ \\
     \midrule
   \cite{shen2020deep}& cube & $f \in C([0,1]^d)$ &$\mathcal{O}(N)$&$\mathcal{O}(L)$& $ \mathcal{O}\left(19\sqrt{d}w_f(N^{-2/d}L^{-2/d})\right)$ \\
   \midrule
\cite{yarotsky2017error} & cube & $f\in W^{n,\infty}([0,1]^D)$ &$\mathcal{O}\left(\ln{(1/\varepsilon)}\right)$&$\mathcal{O}\left(\varepsilon^{-d/n}\ln{(1/\varepsilon)}\right)$& $\varepsilon$ \\
   \midrule
   \cite{kidger2020universal}& - & $f\in C(X) \text{ or } L^{p}(X)$ &-&-& -  \\
   \midrule  \cite{hanin2017approximating}& cube & $f\in C([0,1]^d)$ &$d+3$&$\frac{2\cdot d!}{\omega_f(\epsilon)^d}$& $\varepsilon$  \\
   \midrule   \cite{hanin2019universal} & cube & 
$f\in C^{\infty}([0,1]^d)$ &$d+2$&$L$& $\mathcal{O}\left(L^{-1/d}\right)$  \\
   \midrule
   \cite{park2020minimum} & cube & $f\in L^p(\mathbb{R}^d)$ &$N\geq d+1$ &-&- \\
   \midrule
   \textbf{\cref{thm:MainThmContinuous}}
   & \textbf{polytope} & $\boldsymbol{f \in C(K)}$ & $\mathcal{O}\left(\max\left
    \{N^3,N^{d-1}\log{N}\right\}\right)$&$\mathcal{O}(\log N)$& $\mathcal{O}\left(\tilde{\omega}_K\left(f,n^{-1}\right)+n^{-\alpha}\right)$  \\
   \midrule
   \textbf{\cref{thm:MainThmAnalytic}} 
   & \textbf{polytope} & $\boldsymbol{f}$ \textbf{analytic} & $\mathcal{O}\left(\max\left
    \{N^3,N^d\right\}\right)$&$\mathcal{O}(N)$&$\mathcal{O}_{\epsilon}\left((h+\sqrt{1+h^2})^{-N/\sqrt{d}}\right)$  \\
   \bottomrule
   
 \end{tabular}
 }
\label{tab:universal_approximation}
\vspace{-0.3cm}
\end{table}

\section{Approximation on Polytope}
In this section, we will prove \cref{thm:MainThmContinuous}.

\subsection{Notations and definitions}
\label{prem}
First, we summarize the main notation in this paper as follows:
\begin{itemize}
    \item Vectors and matrices are bold-faced. The function in a bold font denotes a vector. For example, $\boldsymbol{\cos}(\boldsymbol{x})=(\cos{x_1},\ldots,\cos{x_d})$ for any $\boldsymbol{x}\in\mathbb{R}^d$.





    \item In $\mathbb{R}^{d}$, we call a closed set  $K \subset \mathbb{R}^{d}$ a convex polytope, if it is the convex hull of finitely many points. $K$ is $d$-dimensional if it has an inner point, which we shall always assume. 


\item Parallelepiped is a 3D shape whose faces are all parallelograms. Hyper-parallelepiped describes the parallelepiped in high dimensions. For simplicity, we omit the prefix “hyper" when there is no ambiguity in the expression. Given a parallelepiped $E$ and a positive number $\lambda$, let $E^\lambda$ denote the dilation of $E$ by a factor $\lambda$ made its center.


\item In $\mathbb{R}^d$, we use $Q$ to denote the $d$-dimensional hypercube $[-1,1]^d$. For any $\delta\in(0,1)$ we use $Q_\delta$ to denote $[-1+\delta,1-\delta]^d$. Specially in one dimension, we use $I$ and $I_\delta$ to denote $[-1,1]$ and $[-1+\delta,1-\delta]$, respectively.

\item $\Pi_{n}^d$ is the set of all polynomials in $\mathbb{R}^d$ with a total degree no more than $n$:
 \begin{equation*}
     \Pi_{n}^d=\left\{\sum_{0\leq j_1+\cdots+j_d\leq n}a_{j_1,\ldots,j_d}x_1^{j_1}\cdots x_d^{j_d},a_{j_1,\ldots,j_d}\in\mathbb{R}\right\},
 \end{equation*}
 and $\pi_{n}^{T}$ is the set of all trigonometric polynomials in $\mathbb{R}$ of  degree no more than $n$:
 \begin{equation*}
     \pi_{n}^{T}=\left\{\sum_{j=0}^{n}a_j\cos(jx)+b_j\sin(jx),a_j,b_j\in\mathbb{R}\right\}.
 \end{equation*}



\end{itemize}

\begin{definition}[Classical ReLU network]
    Let $\sigma:\mathbb{R}\to\mathbb{R},x\mapsto \max\left\{0,x\right\}$ be the ReLU activation. A classical ReLU network $g(\boldsymbol{x}):\mathbb{R}^{d} \rightarrow \mathbb{R}$ with widths  $N_{1}, \ldots, N_{L}$ of $L$  hidden layers is given by
        \begin{equation*}
            g(\boldsymbol{x})=\mathcal{A}_L\circ \sigma \circ\mathcal{A}_{L-1}\circ\sigma\circ\cdots\circ\sigma\circ\mathcal{A}_1\circ\sigma\circ\mathcal{A}_0,
        \end{equation*}
        where $\mathcal{A}_{j}:\mathbb{R}^{N_j}\to \mathbb{R}^{N_{j+1}}$ is an affine transform. Without loss of generality, we let $N_0=d,N_{L+1}=1$.
\end{definition}

\begin{definition}[Intra-linked ReLU network]
 Let $\sigma:\mathbb{R}\to\mathbb{R},x\mapsto \max\left\{0,x\right\}$ be the ReLU activation function. An ReLU network with intra-layer linked shortcuts $h(\boldsymbol{x}):\mathbb{R}^{d} \rightarrow \mathbb{R}$ with widths $N_{1}, \ldots, N_{L}$ of $L$  hidden layers is given by
        \begin{equation*}
            h(\boldsymbol{x})=\mathcal{A}_L\circ \boldsymbol{\Sigma}_L \circ\mathcal{A}_{L-1}\circ\boldsymbol{\Sigma}_{L-1}\circ\cdots\circ\boldsymbol{\Sigma}_2\circ\mathcal{A}_1\circ\boldsymbol{\Sigma}_1\circ\mathcal{A}_0,
        \end{equation*}
        where $\boldsymbol{\Sigma}_j:\mathbb{R}^{N_j}\to \mathbb{R}^{N_j},\boldsymbol{x}=(x_1,\ldots,x_{N_j})^T\mapsto \boldsymbol{\Sigma}_j(\boldsymbol{x})=\left( 
y_1,\ldots,y_{N_j} \right)^T,j=1,\ldots ,L$ is defined by
        \begin{equation*}        y_1=\sigma(x_1),\quad 
        y_k=\sigma(x_k+\langle\boldsymbol{W}_{k}^{j},(y_1,\ldots,y_k)^T\rangle)                 
        \end{equation*}
        where $\boldsymbol{W}^j=\left[{\boldsymbol{W}_1^j}^T,\ldots,{\boldsymbol{W}_{N_j}^j}^T\right]\in\mathbb{R}^{N_j\times N_j}$ is the learnable weights of intra links of the $j$-th layer, and is strictly lower triangular. Especially, when $\boldsymbol{W}^j=\boldsymbol{0}$, the network degrades as a classical feedforward ReLU network. Figure~\ref{Comparison:feedforward_and_intra_linked} showcases the topological and computational differences between feedforward and intra-linked neural networks.
         
\end{definition}

Later, we will show that \textbf{an intra-linked network can achieve much higher approximation efficiency than a normal feedforward network}.

\begin{figure*}[ht!]
  \centering
    \begin{subfigure}{(a)}
      \centering   
      \includegraphics[width=1\linewidth]{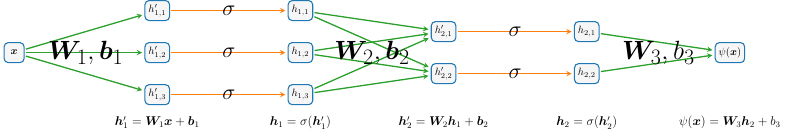}

    \end{subfigure}
    
    \begin{subfigure}{(b)}
      \centering   
      \includegraphics[width=\linewidth]{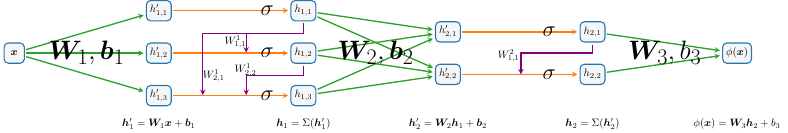}
    \end{subfigure}
\caption{(a) A classical ReLU network of widths $3$ and $2$ for two hidden layers. (b) Intra-linked ReLU network of width $3$ and $2$ for two hidden layers.}
\label{Comparison:feedforward_and_intra_linked}
\end{figure*}

\subsection{Ideas for proving
\cref{thm:MainThmContinuous}}

The key idea of extending the approximation over hypercubes to a general polytope is based on the fact that each simple polytope can be decomposed into a union of a group of parallelepipeds (\cref{ParallelepipedCovering}). The main difficulty is ensuring that the approximation over each parallelepiped does not interfere with each other and eliminating the value outside the domain. To this end, we adapt techniques in the theory of polynomial approximation \cite{Ditzian1996Polynomial,Totik2014PolynomialAO} into the construction of neural networks.

We mainly follow three steps in proving \cref{thm:MainThmContinuous}:
\begin{itemize}
    \item First, we construct a polynomial in $\mathbb{R}^d$ to approximate continuous functions on $Q$. Although such an approximation has been well studied in polynomial approximation theory, it is crucial to provide polynomials with explicit coefficients for further estimate.

    \item Second, we construct ReLU networks with intra-layer links to approximate polynomials efficiently, which, together with the first step, gives an approximation to $C(Q)$ functions by ReLU networks. Besides, we also demand that the approximation is on $Q$, which ensures the local approximations have no cross-talk.

    \item Third, because each simple polytope has a finite number of parallelepipeds covering it, we extend the approximation over a cube $Q$ to a parallelepiped $K$. Then, approximation over simple polytopes is accomplished.
\end{itemize}

We formulate the main results in the above steps as \cref{PolynomialConstruction,MainThmCube,Thm_Global}, whose proofs will be given in Sections \ref{Sec_KPM}, \ref{Sub_PolyApprox}, and \ref{Sec_proof}, respectively. Next, we prove \cref{thm:MainThmContinuous} based on \cref{PolynomialConstruction,MainThmCube,Thm_Global}.

 {
\begin{theorem}\label{PolynomialConstruction}
Let $f$ be continuous on the cube $Q=[-1,1]^d$. Then for any $N\in \mathbb{Z}^{+}$, there exists a polynomial $P_n \in \Pi_n^d$, which can be formulated as a Chebyshev series $P_n=\sum_{j_1,\ldots,j_d=1}^{n}a_{j_1,\ldots,j_d}T_{j_1}(x_1)\cdots T_{j_d}(x_d)$ with bounded coefficients $a_{j_1,\ldots,j_d}$, such that
\begin{equation*}
    \left\|f-P_n\right\|_{L^{\infty}(Q)}\leq C \tilde{\omega}_{Q}\left(f,\frac{1}{n}\right),
\end{equation*}
where $C$ is a constant independent of $f$ and $n$. 
\end{theorem}
}

 {
\begin{theorem}\label{MainThmCube}
Let $f$ be continuous on the cube $Q=[-1,1]^d$. Then for any $N\in \mathbb{Z}^{+}$ and $\alpha>0$, there exists a  function $\tilde{f}:\mathbb{R}^d\to \mathbb{R}$ vanishing outside $Q$ that can be implemented by a ReLU network of width $\mathcal{O}\left(\max\left
    \{N^3,N^{d-1}\log{N}\right\}\right)$ and depth $\mathcal{O}(\log N)$ such that for any $0<\delta<1$
	\begin{equation}
		\left\|\tilde{f}-f\right\|_{L^{\infty}(Q_\delta)}\leq \mathcal{O}\left(\tilde{\omega}_Q\left(f,N^{-1}\right)+N^{-\alpha}\right)
	\end{equation}
and that $\text{supp}(\tilde{f})\subseteq Q$
\end{theorem}
}

\begin{theorem}\label{Thm_Global}
     Let $K\subseteq\mathbb{R}^d$ covered by $\left\{K_{j}\right\}_{j=1}^{k}$ and let $f:K\to\mathbb{R}$ be continuous. Assume that there exists a collection of $\hat{f}_{j}$ satisfying
\begin{equation}
    \left\| f-\hat{f}_{j} \right\|_{L^{\infty}\left(K_j\right)}\leq err \quad \text{and} \quad\left\|\hat{f}_j\right\|_{L^{\infty}(K_j)}\leq M,
\end{equation}
and that each $\hat{f}_{j}$ can be implemented by a ReLU network with width $W$ and depth $D$. Then for any $N\in\mathbb{Z}^+$, there exists a function $\widetilde{f}$ implemented by a ReLU network with width $\max\left\{8(N+1)+2(k-2),kW+4(k-1)d\right\}$ and depth $\max\{D,d\}+2(k-1)$ such that
\begin{equation*}
    \left\|f-\tilde{f}\right\|_{L^{\infty}\left(K\right)}\leq \frac{3(k+2)^2}{4}M\cdot 2^{-2N}+err
\end{equation*}
and accordingly
\begin{equation*}
    \text{supp}\left(\tilde{f}\right)\subseteq \bigcup_{j=1}^{k}\text{supp}\left(\hat{f}_{j}\right).
\end{equation*}
\end{theorem}

\subsection{Proof of 
\cref{thm:MainThmContinuous}}\label{Approx_over_poly}

To prove \cref{thm:MainThmContinuous}, we first show the correctness of \cref{ParallelepipedCovering}, which is a weaker form of the result in Section 4 of \cite{Totik2014PolynomialAO}.
\begin{lemma}\label{ParallelepipedCovering}
    Let  $K$  be a polytope with vertices $v_{1}, \ldots, v_{m}$. Then there exists a group of parallelepipeds  $\left\{K_{j}\right\}_{j=1}^{k}$, whose edges are paralleled to some edges of $K$, such that
\begin{equation*}
    K=\bigcup_{j=1}^{k}K_{j}.
\end{equation*}
\end{lemma}

\begin{proof}
    Without loss of generality, we may assume that $K$ is simple, \textit{i.e.}, each vertex of $K$ is joined to other vertices by exactly $d$ edges. Let $K^{v}\in K$ be the parallelepiped with the vertex $v$ and edges parallel to edges adjacent to $v$. For every $y \in K$, there exists a parallelepiped  $K(y) \subset K$  containing $y$ such that $K(y)$ is a translation of some $K^{v}$. Thus,  $K=\cup_{y \in K} K(y)$. We claim that there are finitely many of these $K(y)$ covering $K$, \textit{i.e.}, $K=\cup_{i=1}^{k} K\left(y_{i}\right) $. Among $K$, vertices $v_{1}, \ldots, v_{m}$ are covered by $m$ parallelepipeds, say  $K\left(v_{1}\right), \ldots, K\left(v_{m}\right) $. Next, we show that this finite covering property holds for edges (1-dimensional faces). Then the result holds for all faces and $K$. Without loss of generality, we assume  $E=v_{1} v_{2}$ is an edge, and $E$  cannot be covered by $K\left(v_{1}\right) \cup K\left(v_{2}\right)$. So, we can find a closed segment  $E^{\prime} $ properly containing $E$ such that $E \subset K\left(v_{1}\right) \cup K\left(v_{2}\right) \cup E^{\prime}$. Due to the construction of covering,  $\left\{\text{int}\;K(y)\right\}_{y \in E^{\prime}}$  (the interior of  $K(y)$) is a family of open sets covering  $E^{\prime}$. By the compactness of  $E^{\prime}$, $E^{\prime}$  can be covered by finitely many $\left\{K\left(y_{i}^{\prime}\right)\right\}_{i=1}^{k^{\prime}}$. Then  $E \subset \cup_{i=1}^{k^{\prime}} K\left(y_{i}^{\prime}\right) \cup   K\left(v_{1}\right) \cup K\left(v_{2}\right)$ .
\end{proof}

We next prove \cref{thm:MainThmContinuous} by assuming \cref{ParallelepipedCovering}, \cref{PolynomialConstruction}, \cref{MainThmCube}, and \cref{Thm_Global} are true. Proofs
of Theorems \ref{PolynomialConstruction}, \ref{MainThmCube}, and \ref{Thm_Global} are detailed in Sections 3-5, respectively.

\begin{proof}[Proof of \cref{thm:MainThmContinuous}]
 {
    Let $\{ K_j \}_{j=1}^k$ be some parallelepiped covering of $K$. We use $\Psi_i:\mathbb{R}^d\to\mathbb{R}^d$ to denote the affine transfrom that maps $K_j\to Q$ for all $j=1,\ldots,k$. For each $j=1,\ldots,k$, $f\circ \Psi_j^{-1}$ is continuous over $Q$. Note that for any $0<\lambda<1$, $\Psi_j(K_j^{\lambda})=Q_{\delta},\delta=1-\lambda$. Then by \cref{MainThmCube}, there is a function $\hat{f}_j$ supporting in $Q$, which is implemented by a ReLU network with width $\mathcal{O}\left(\max\left
    \{N^3,N^{d-1}\log{N}\right\}\right)$ and depth $\mathcal{O}(\log N)$ such that
	\begin{equation}
		\left\|\tilde{f}-f\right\|_{L^{\infty}(Q_\delta)}\leq \mathcal{O}\left(\tilde{\omega}_Q\left(f,N^{-1}\right)+N^{-\alpha}\right).
	\end{equation}
\begin{equation}
\begin{split}
    \left\|f\circ\Psi_j^{-1}-\hat{f}\right\|_{L^{\infty}\left(Q_{\delta}\right)}&=\left\|f-\hat{f}\circ \Psi_j\right\|_{L^{\infty}\left(K_j^{\lambda}\right)}\\
    &\leq \mathcal{O}\left(\tilde{\omega}_{Q}\left(f\circ\Psi_j,N^{-1}\right)+N^{-\alpha}\right)\\
    &=\mathcal{O}\left(\tilde{\omega}_Q\left(f,N^{-1}\right)+N^{-\alpha}\right),
\end{split}
\end{equation}
where we use the fact that $\tilde{\omega}_{\Psi(K)}\left( 
f,t \right)=\tilde{\omega}_{K}\left(f\circ\Psi,t\right)$ for any polytope $K$ and $t>0$.
}

 {
Note that the affine transform is performed in the pre-activation step at the beginning. Hence it takes the same depth and width to implement $\hat{f}_j$ as $\hat{f}_j\circ\Psi_j$ for all $j=1,\ldots,k$. Then combining \cref{Thm_Global} and \cref{MainThmCube}, we obtain the result immediately.
}
\end{proof}

\begin{remark}
 {
    Please note that covering a polytope with parallelepipeds is not essential. Our approach has the flexibility to choose other suitable covering. For example, we can also consider simplex covering as follows. By Lemme 14 in \cite{Fan2023Quasi}, there is a ReLU network of width $(d+1)d^2$ and depth $d+1$ such that it can implement a linear function on a $d$-dimensional simplex $S$ and vanishes outside $S$ except for a trifling region of arbitrarily small. By composing such subnetworks, we restrict the constructed function to simplex $S$. Thus, we can rewrite \cref{MainThmCube} and \cref{Thm_Global} based on simplex decomposition. }
\end{remark}


\section{Proof of Theorem~\ref{PolynomialConstruction}}\label{Sec_KPM}

In the theory of polynomials approximate, one of the most compelling problems is to find a polynomial $q_n$ satisfying a Jackson-type error bound, \textit{i.e.}, $\inf_{q_n\in\Pi_n^d}\left\|f-q_n\right\|_{L^{p}}\leq M_{r,K} \omega_{K}^{r}\left(f,\frac{1}{n}\right)$ for a continuous function $f$ defined on a polytope $K$. The first Jackson-type inequality dates back to \cite{Jackson1912Transactions}. Later, although numerous studies extended the original result to higher-order moduli of smoothness and special polytope domains \cite{Butzer1978On,Ditzian1996Polynomial,Lorentz1968Approximation,ditzian2012moduli}, the Jackson-type estimate over a general polytope remains unresolved for decades, until Totik's work \cite{Totik2014PolynomialAO,Totik2020} in recent years.

\textit{What is special in Theorem~\ref{PolynomialConstruction} is it provides the explicit polynomial construction, and the coefficients of the polynomial are explicitly derived, which can facilitate the error analysis and the downstream network construction.} The key to disentangling coefficients and the degree is to introduce a special kind of Jackson-type kernels, which is inspired by Jackson's work \cite{Jackson1912Transactions} whose main technique is convolution. To the best of our knowledge, it is the first time that the Jackson-type kernel method is used in deep learning approximation theory, which goes beyond the commonly used paradigm of Taylor expansion.
The proof of Theorem~\ref{PolynomialConstruction} is threefold:

\textbf{Step 1: Approximate univariate $2\pi$-periodic functions}. For a $2\pi$-periodic function $f:\mathbb{R}\to \mathbb{R}$, we construct a Jackson-type kernel $J_n\in\Pi_n^T$, which under some conditions fulfills 
     $$\sup_{x\in [-\pi,\pi]}\left|(f*J_n)(x)-f(x)\right|\leq C\omega\left(f,\frac{1}{n}\right),$$
for some constant $C$ independent with $f$ and $n$.
     Note that $(f*J_n)\in\Pi_n^T$,
    \textit{e.g.}, $\cos{k(x-t)}=\cos{kx}\cos{kt}+\sin{kx}\sin{kt}$, and $\int_{-\pi}^{\pi}f(x-t)J_n(t)dt=\int_{-\pi}^{\pi}f(x)J_n(x-t)dt$.

\begin{definition}[Jackson-type kernel]
\label{def_kernel}
    A series of trigonometric polynomials $\left\{ J_n(t) \right\}_{n=0}^{\infty}$ is said to be a Jackson-type kernel if conditions (a)-(d) are satisfied for each $J_n(t)$:  
\begin{flalign}
    &\ \ \ \ ~\nonumber\text{(a)} \text{(Positiveness)} J_n(t)\geq 0;\\
    &\ \ \ \ ~\nonumber\text{(b)} \text{(Evenness)}\;J_n(t)=J_N(-t);&\\
    &\ \ \ \ ~\nonumber\text{(c)} \text{(Normality)}\;\int_{-\pi}^{\pi}J_n(t)dt=1;&\\
    &\ \ \ \ ~\nonumber\text{(d)} \text{(Moment condition)}\;n^2\int_{-\pi}^{\pi}t^2 J_n(t)dt\leq M,\;\text{for some constant }M>0.&
\end{flalign}

\end{definition}

Now, we use the Jackson-type kernel to construct a polynomial approximation for a univariate periodic continuous function. 

\begin{lemma}\label{Jackson_Periodic}
    If $J_n(t)$ is a Jackson-type kernel, then for any continuous $2\pi$-periodic $f:\mathbb{R}\to\mathbb{R}$,
\begin{equation*} 
    \left\| f-f*J_n  \right\|_{L^{\infty}(\mathbb{R})}\leq (M+1)\cdot\omega\left(f, \frac{1}{n} \right),
\end{equation*}
where $C$ is the constant that satisfies the moment condition of the Jackson-type kernel.
\end{lemma}

\begin{proof} Leveraging the normality of the Jackson-type kernel, we have
\begin{align}
    \nonumber\left|f(x)-\int_{-\pi}^{\pi}f(x-t)J_n(t)dt \right|&=\left|\int_{-\pi}^{\pi}\left(f(x)-f(x-t)\right)J_n(t)dt \right|\\
    \nonumber&\leq \int_{-\pi}^{\pi}\omega(f,|t|)J_n(t)dt\\
    \nonumber&\leq \omega\left(f, \frac{1}{n}\right)\int_{-\pi}^{\pi}(1+n|t|)J_n(t)dt \\ 
    \nonumber&\leq (M+1)\omega\left(f, \frac{1}{n} \right),
\end{align}
where the second inequality is from the fact that
\begin{equation*}
    \omega(f,|t|)=\omega\left(f,n|t|\cdot \frac{1}{n}\right)\leq (1+n|t|)\omega\left(f, \frac{1}{n}\right),
\end{equation*}
and the last inequality follows the moment condition.
\end{proof}

Here, we apply a kernel that is explicitly formulated as a trigonometric polynomial. Now, we show that this kernel is a Jackson-type kernel.

\begin{proposition}\label{explicit_kernel}
    Let
\begin{equation*}
J_n(t)=\sum_{k=0}^{n}\widehat{a}_{k}\cos{kt}=\frac{1}{\pi}\left|\sum_{k=0}^{n}a_{k,n}e^{ikt} \right|^2,
\end{equation*}
where 
\begin{equation*}
    a_{k,n}=\left( 2\sum_{j=0}^{n}\sin^2{\frac{j+1}{n+2}\pi} \right)^{-\frac{1}{2}} \sin{\frac{k+1}{n+2}\pi}.
\end{equation*}
Then $J_n(t)$ is a Jackson-type kernel with a constant $M=\frac{\pi^2}{2}$ satisfying the moment condition. Besides, $\left|\widehat{a}_{k}\right|\leq A,k=0,1,\ldots,n$, for some absolute constant $A$.
\end{proposition}

\begin{proof}
For the above kernel, conditions $(a)$ and $(b)$ defining the Jackson-type kernel are obviously satisfied. The condition $(c)$ follows from a direct calculus after expanding $J_n(t)$ as
\begin{align}
    \nonumber \frac{1}{\pi}\left|\sum_{k=0}^{n}a_{k,n}e^{ikt} \right|^2 &=\frac{1}{\pi}\left( \sum_{k=0}^{n}a_{k,n}e^{ikt} \right)\left( \sum_{h=0}^{n}a_{h,n}e^{-iht} \right)\\
    \nonumber &=\frac{1}{\pi}\sum_{k=1}^{n}\sum_{h=1}^{n}a_{k,n} a_{h,n} e^{i(k-h)t}\\
    \nonumber&=\frac{1}{2\pi}+\frac{2}{\pi}\sum_{1\leq h<k\leq n}a_{k,n} a_{h,n} \cos{(k-h)t}.
\end{align}

 For the moment condition $(d)$, we have
\begin{equation}\label{MomentCond}    
\begin{split}
        \int_{-\pi}^{\pi}|t|J_n(t)dt &=\int_{-\pi}^{\pi}|t|\sqrt{J_n(t)}\cdot\sqrt{J_n(t)}dt\\
    &\leq \left( \int_{-\pi}^{\pi}|t|^2 J_n(t)dt \right)^{\frac{1}{2}}\left( \int_{-\pi}^{\pi}J_n(t)dt \right)^{\frac{1}{2}}\\
    &\leq\left( \int_{-\pi}^{\pi}\frac{\pi^2}{2}(1-\cos{t}) J_n(t)dt \right)^{\frac{1}{2}}\\
    &=\pi\left(\frac{1-2\sum_{k=0}^{n-1}a_{k,n}a_{k+1,n}}{2} \right)^{\frac{1}{2}},
\end{split}
\end{equation}
where the first and second inequalities follow from the Cauchy inequality and $|t|^{2}\leq \frac{\pi^2}{2}(1-\cos{t})$ when $t\in [-\pi, \pi]$, respectively. 

Then we estimate $\sum_{k=0}^{n-1}a_{k,n}a_{k+1,n}$. For simplicity of notation, we denote $A_n=\sqrt{2 \sum_{k=0}^{n}\sin^2{\frac{k+1}{n+2}\pi}} $. We have
\begin{align}
    \nonumber\sum_{k=0}^{n-1}a_{k,n}a_{k+1,n}&=\frac{A_n^2}{2}\left( \sum_{k=0}^{n}\sin{\frac{k+1}{n+2}\pi} \sin{\frac{k+2}{n+2}\pi}+\sum_{k=0}^{n}\sin{\frac{k}{n+2}\pi} \sin{\frac{k+1}{n+2}\pi} \right)\\
    \nonumber&=\frac{A_n^2}{2}\sum_{k=0}^{n}\left( \sin{\frac{k+2}{n+2}\pi}+ \sin{\frac{k}{n+2}\pi}\right)\sin{\frac{k+1}{n+2}\pi}\\
    \nonumber&=\frac{A_n^2}{2}\cos{\frac{\pi}{n+2}}\sum_{k=0}^{n}\sin^2{\frac{k+1}{n+2}\pi}\\
    \nonumber&=\frac{1}{2}\cos{\frac{\pi}{n+2}}.
\end{align}

Hence, we obtain the moment estimate
\begin{equation}
    \nonumber n\int_{-\pi}^{\pi}|t|J_{n}(t)dt\leq n\pi\sqrt{\frac{1-\cos{\frac{\pi}{n+2}}}{2}}
    =n\pi\sin{\frac{\pi}{2n+4}}\leq \frac{\pi^2}{2}.
\end{equation}

Finally, we estimate $\left|\widehat{a}_{k}\right|$. Recall that $\widehat{a}_{0}=\frac{1}{2\pi}$ and 
\begin{equation}\label{a_hat}
    \widehat{a}_{k}=\frac{2}{\pi}\sum_{j=0}^{n-k}a_{j,n}a_{j+k,n}
    =\frac{2}{\pi}\frac{\sum_{j=0}^{n-k}\sin{\frac{j+1}{n+2}\pi}\sin{\frac{j+k+1}{n+2}\pi}}{\sum_{j=0}^{n}\sin^{2}{\frac{j+1}{n+2}\pi}}.   
\end{equation}
On the one hand, we have
\begin{equation}
\label{numerator}
\begin{split}
    \sum_{j=0}^{n}\sin^{2}{\frac{j+1}{n+2}\pi}&=2\sum_{j=0}^{\lfloor \frac{n+2}{2}\rfloor}\sin^{2}{\frac{j+1}{n+2}\pi}\\
    &\geq 2\sum_{j=0}^{\lfloor \frac{n+2}{2}\rfloor}\frac{4}{\pi^2}\left(\frac{j+1}{n+2}\pi\right)^2\\
    &=\frac{8}{(n+2)^2}\cdot\frac{\lfloor \frac{n+2}{2}\rfloor \left(\lfloor \frac{n+2}{2}\rfloor+1\right) \left(2\lfloor \frac{n+2}{2}\rfloor+1\right)}{6}\\
    &\geq \frac{4}{3(n+2)^2}\frac{n+1}{2} \left(\frac{n+1}{2}+1\right) \left(n+2\right)\triangleq S^1_n.
\end{split}
\end{equation}

On the other hand, 
\begin{equation}
\label{denominator}
\begin{split}
    \sum_{j=0}^{n-k}\sin{\frac{j+1}{n+2}\pi}\sin{\frac{j+k+1}{n+2}\pi}&\leq \pi^2\sum_{j=0}^{n-k}\frac{j+1}{n+2}\frac{j+k+1}{n+2}\\
    &\leq\frac{\pi^2}{(n+2)^2}\sum_{j=1}^{n}j^2+n\sum_{j=1}^{n}j\\
    &=\frac{\pi^2}{(n+2)^2}\left(\frac{n(n+1)(2n+1)}{6}+\frac{n^2(n+1)}{2}\right)\triangleq S^2_n.
\end{split}
\end{equation}

Since both \eqref{numerator} and \eqref{denominator} are $\Theta(n)$, there is some constant $A>0$ such that $\left|\widehat{a}_{k}\right|\leq A$. Furthermore, we can take $A=5\pi$. Actually, one can show that $S^2_n/S^n_1$ is increasing after some simple but cumbersome calculation. Then the result follows from $\lim_{n\to\infty}S^n_2/S^n_1=5\pi^2/2$. 
\end{proof}

\begin{remark}
\label{rmk_Larentz}
    Actually, we have shown in \cref{MomentCond} that the Jackson-type kernel we construct in \cref{Jackson_Periodic} also satisfies the second order moment condition. It was generalized by Lorentz \textit{et al}. \cite{Lorentz1968Approximation} that the \textit{a.k.a.} Lorentz kernel $L_{n,r}(t)=\gamma_{n,r}\left(\frac{\sin nt/2}{\sin t/2}\right)^{2r}$ satisfies the $k$-th order moment condition $\int_{-\pi}^{\pi}|t|^k L_{n,r}(t)dt=\mathcal{O}(n^{-k}),k\leq 2r-2$) as well as condition (a)-(c) in \cref{def_kernel}, which can improve the error bound to higher order moduli of smoothness. That is, we have for any $2\pi$-periodic function $f$
\begin{equation*}
\begin{split}
    \left\|f-f\ast J_n\right\|_{L^{\infty}(\mathbb{R})}&\leq \bar{M} \omega^r \left(f,\frac{1}{n}\right)\\
    &= \bar{M}\sup_{x\in \mathbb{R},h\leq \frac{1}{n}}\left| \Delta_h^r f(x) \right|,
\end{split}
\end{equation*}
where $\Delta_h^r f(x)$ is the $r$-th symmetric difference recursively defined by
\begin{equation*}
    \Delta_h^r f(x)=\begin{cases}
        \Delta_h^{r-1}f\left(x+\frac{h}{2}\right)-\Delta_h^{r-1}f\left(x-\frac{h}{2}\right), & r>1\\
        f\left(x+\frac{h}{2}\right)-f\left(x-\frac{h}{2}\right),&r=1.
    \end{cases}
\end{equation*}
\end{remark}


\textbf{Step 2: Approximate univariate functions on $[-1,1]$}: Note that $J_n$ is even. If $f$ is an even function and $2\pi$ periodic, $f*J_n$ is a combination of terms like $\int f(t)\cos(k(x-t))dt=\int f(t)\cos(kt)dt \cdot \cos(kx)-\int f(t)\sin(kt)dt \cdot \sin(kx)=\int f(t)\cos(kt)dt \cdot \cos(kx)$, which is simply the combination of Chebyshev polynomials $T_k(\cos x)$. Hence for a univariate function $f:I\to\mathbb{R}$, such a polynomial approximation is obtained by considering $f\circ \cos=f(\cos x)$. We use $\mathcal{T}_n:C(I)\to \Pi_n$ to denote this operation. Next, we show the polynomial approximation for the univariate continuous function on $[-1,1]$ by considering a composition with the cosine function.

\begin{proposition}[Jackson inequality]\label{JacksonThm}
    There exists an operator $\mathcal{T}_{n}:C(I)\to\Pi_{n}$ with $\left\|\mathcal{T}_{n}\right\|\leq 1$ such that for any $f\in C(I)$, 
\begin{equation*}
    \left\|f-\mathcal{T}_{n}(f)\right\|_{L^{\infty}(I)}\leq C \tilde{\omega}_I\left(f,\frac{1}{n}\right),
\end{equation*}
where $C$ is an absolute constant independent with $f$ and $n$.
\end{proposition}

\begin{proof}
    We consider $f(\cos{\theta})$, which is $2\pi$-periodic and even. Let $J_n$ be the Jackson-type kernel in \cref{explicit_kernel}. Then  $(f\circ\cos)*J_n\in \Pi_n^T$ has no sine terms. We set 
\begin{equation*}
\begin{split}    
    \mathcal{T}_{n}(f)(x)&=\int_{-\pi}^{\pi}f\left(\cos(\arccos{x}-t)\right)J_n(t)dt\\
    &=((f\circ\cos)*J_n)(\arccos{x}),
\end{split}
\end{equation*}
which gives an algebraic polynomial of degree $n$. Then we have
\begin{equation}
\begin{split}
    \left\|\mathcal{T}_{n}f\right\|_{L^{\infty}(I)}&=\max_{x\in[-1,1]}\left|((f\circ\cos)*J_n)(\arccos{x})\right|\\
    &=\max_{\theta\in[-\pi,\pi]}\left|((f\circ\cos)*J_n)(\theta)\right|\\
    &\leq\max_{\theta\in[-\pi,\pi]}\int_{-\pi}^{\pi}\left|(f\circ\cos)(\theta-t)J_n(t)\right|dt \\
    &\leq \left\|f\circ\cos\right\|_{L^{\infty}([-\pi,\pi])}\int_{-\pi}^{\pi}\left|J_n(t)\right|dt\\
    &=\left\|f\right\|_{L^{\infty}(I)}.
\end{split}   
\end{equation}


Hence, it follows directly from \cref{Jackson_Periodic,explicit_kernel} that
\begin{equation*}
    \left\|f-\mathcal{T}_{n}(f)\right\|_{L^{\infty}(I)}=\left\|f\circ\cos-(f\circ\cos)*J_n\right\|_{L^{\infty}(\mathbb{R})}\leq \left(\frac{\pi^2}{2}+1\right)\omega\left(f\circ\cos,\frac{1}{n}\right).
\end{equation*}

Finally, using the fact that 
\begin{equation*}
\label{eq_EquivMod}
    \omega(f\circ \cos,t)\leq C^{\prime} \tilde{\omega}_I(f,t)
\end{equation*}
for some constant $C^{\prime}$, which can be found in Theorem 2 and Remark 1 in \cite{dyuzhenkova1995remark}, we immediately obtain
\begin{equation*}
    \left\|f-\mathcal{T}_{n}(f)\right\|_{L^{\infty}(I)}\leq C \tilde{\omega}_I\left(f,\frac{1}{n}\right)
\end{equation*}
with $C=\left(\frac{\pi^2}{2}+1\right)C^{\prime}$.
\end{proof}

We give an explicit formulation of $\mathcal{T}_{n}(f)$ in \cref{JacksonThm}. Let $K_n$ be the Jackson-type kernel defined in \cref{explicit_kernel}. Then 
\begin{equation}
\begin{split}
((f\circ\cos)*J_n)(\theta)&=\sum_{k=0}^{n}\hat{a}_{k}\int_{-\pi}^{\pi}f(\cos{t})\cos{(k(\theta-t))}dt\\
    &=\sum_{k=0}^{n}\hat{a}_{k}\int_{-\pi}^{\pi}f(\cos{t})\cos{kt}dt\cos{k\theta}.
\end{split}
\end{equation}
Using the fact $\cos{k\theta}=T_{k}(\cos{\theta})$, we obtain
\begin{equation}\label{explicit_Jackson}
    \mathcal{T}_n(f)(x)=\sum_{k=0}^{n}\hat{a}_{k}\int_{-\pi}^{\pi}f(\cos{t})\cos{kt}dt\;T_{k}(x).
\end{equation}
At the same time, we have shown that $\hat{a}_{k}<5\pi$ for $k=0,\cdots,n$, which means that $\hat{a}_{k}$ is well bounded, and its upper bound is independent of the degree of the polynomial. 

\begin{remark}
    For higher order of modulus of smoothness, \cref{eq_EquivMod} can be extended as
\begin{equation*}
    \omega^r(f\circ \cos,t)\leq C\left(\tilde{\omega}^r_I (f,t)+\left\|f\right\|_{L^{\infty}}n^{-r}\right),
\end{equation*}
(see Theorem 2 in \cite{dyuzhenkova1995remark}).
    Hence, as mentioned in \cref{rmk_Larentz}, using the Larentz kernel $L_{n,r}(t)=\gamma_{n,r}\left(\frac{\sin nt/2}{\sin t/2}\right)^{2r}$, we can extend \cref{MainThmCube} and \cref{thm:MainThmContinuous} to higher order estimate.
\end{remark}

\textbf{Step 3: Approximate multivariate functions on a hypercube}: We extend the univariate case to general multivariate functions over a hypercube. The approximation over $Q$ follows the operator $\mathcal{T}_n$ to each variable $x_1,\ldots,x_d$, which is a linear combination of $\left\{T_{j_1,\ldots,j_d}(\boldsymbol{x})\right\}_{0\leq j_1,\ldots,j_d\leq n}$.

\begin{proposition}[Jackson Theorem on Hypercubes]
\label{Jackon_Hypercube}
    There exists an operator $\mathcal{T}_{n}^{d}:C(Q)\to \Pi_n^d$ such that for any $f\in C(Q)$,
\begin{equation*}
    \left\|f-\mathcal{T}_{n}^{d}(f)\right\|_{L^{\infty}(Q)}\leq Cd\cdot\tilde{\omega}_{Q}\left(f,\frac{1}{n}\right),
\end{equation*}
where $C$ is an absolute constant independent with $f$ and $n$.
\end{proposition}

\begin{proof}
    Let $\mathcal{T}_{k,n}$ be the operator in \cref{JacksonThm} applied on the $k$-th variable of $f$ and set
\begin{equation}\label{T_n^d}
\mathcal{T}_{n}^{d}=\mathcal{T}_{1,n}\circ \cdots \circ \mathcal{T}_{d,n}.
\end{equation}
By the definition of $\tilde{\omega}_{Q}(f,\cdot)$ and \cref{JacksonThm}, we have
\begin{equation}
\begin{split}
    \left\|\mathcal{T}_{n}^{d}(f)-f\right\|_{L^{\infty}(Q)}&\leq \sum_{k=0}^{d-1} \left\|\mathcal{T}_{1,n}\circ\cdots\circ\mathcal{T}_{k,n}(f)-\mathcal{T}_{1,n}\circ\cdots\circ\mathcal{T}_{k+1,n}(f)\right\|_{L^{\infty}(Q)}\\
    &\leq \sum_{k=0}^{d-1}\left\|f-\mathcal{T}_{k+1,n}(f)\right\|_{L^{\infty}(Q)}\\
    &\leq C\sum_{k=1}^{d}\sup_{\boldsymbol{x}\in Q, 0<h\leq \frac{1}{n}}\left|\Delta_{k,h\varphi(x_k)}f(\boldsymbol{x})\right|\\
    &\leq Cd\cdot\tilde{\omega}_{Q}\left(f,\frac{1}{n}\right).
\end{split}
\end{equation}

\end{proof}

Similarly, we can also give an explicit formulation of $\mathcal{T}_{n}^{d}$ in \cref{Jackon_Hypercube}. Actually, due to Eqs. \eqref{explicit_Jackson} and \eqref{T_n^d}, we have
\begin{equation}
    \mathcal{T}_{n}^{d}(f)\left(\boldsymbol{x}\right)=\sum_{j_1=0}^{n}\cdots\sum_{j_d=0}^{n}\widetilde{a}_{j_1,\ldots,j_d}T_{j_{1}}(x_1)\cdots T_{j_{d}}(x_d),
\label{Target_Polynomial}
\end{equation}
where
\begin{equation}\label{coefficient}
    \widetilde{a}_{j_1,\ldots,j_d}=\hat{a}_{j_1}\cdots\hat{a}_{j_d}\int_{[-\pi,\pi]^d}f\left(\boldsymbol{\cos}(\boldsymbol{\theta})\right)\cos{j_1\theta_1}\cdots\cos{j_d\theta_d}d\boldsymbol{\theta}.    
\end{equation}

\section{Proof of \cref{MainThmCube}}\label{Sub_PolyApprox}

Earlier, we show how to approximate a
continuous function via the polynomial obtained by performing the Jackson-type kernel convolution. Here, we construct a ReLU neural network to approximate such a polynomial. In the existing error bounds (see Table~\ref{tab:universal_approximation}), only the depth of a network appears in the exponent part, while the width is at the base. In this work, though the overall scheme we follow in constructing the ReLU network is still based on Yarosky \cite{yarotsky2017error}, \textit{by introducing intra-layer shortcuts into a network, we can construct a ReLU network with width $\mathcal{O}(N)$ and depth $\mathcal{O}(1)$ that can approximate polynomials in $[0,1]^d$ with an error of $\mathcal{O}\left(2^{-2N}\right)$}, which is a major extension of the result $\mathcal{O}\left(N^{-L}\right)$ in \cite{shen2019deep} for ReLU network of width $N$ and depth $L$ and $\mathcal{O}\left(2^{-L}\right)$ in \cite{yarotsky2017error} for fixed-width networks. Considering that Yarosky's method is rather popular in deep learning approximation, the technique of introducing intra-layer shortcuts can be scaled into a plethora of other problems.

\textbf{Step 1: Construction of sawtooth functions.} We construct a ReLU network to implement a series of sawtooth functions $g_{s}:[0,1]\to[0,1]$ \cite{Telgarsky2015RepresentationBO}: 
    $$g_{s}(x)=\begin{cases}
        0,&x=\frac{2k}{2^s},k=0,1,\ldots,2^{s-1},\\
        1,&x=\frac{2k-1}{2^s},k=1,\ldots,2^{s-1},\\
        \text{linear},& [\frac{2k-1}{2^s}, \frac{2k}{2^s} ] ~\mathrm{and} ~[\frac{2k}{2^s}, \frac{2k+1}{2^s}].
    \end{cases}$$   
$g_1(x)=2|x-\frac{1}{2}|+1$, and $g_s$ respects the recursive relation: $g_s=g_1\circ g_{s-1}$.

\begin{figure}[htb!]
\center{\includegraphics[width=0.9\linewidth] {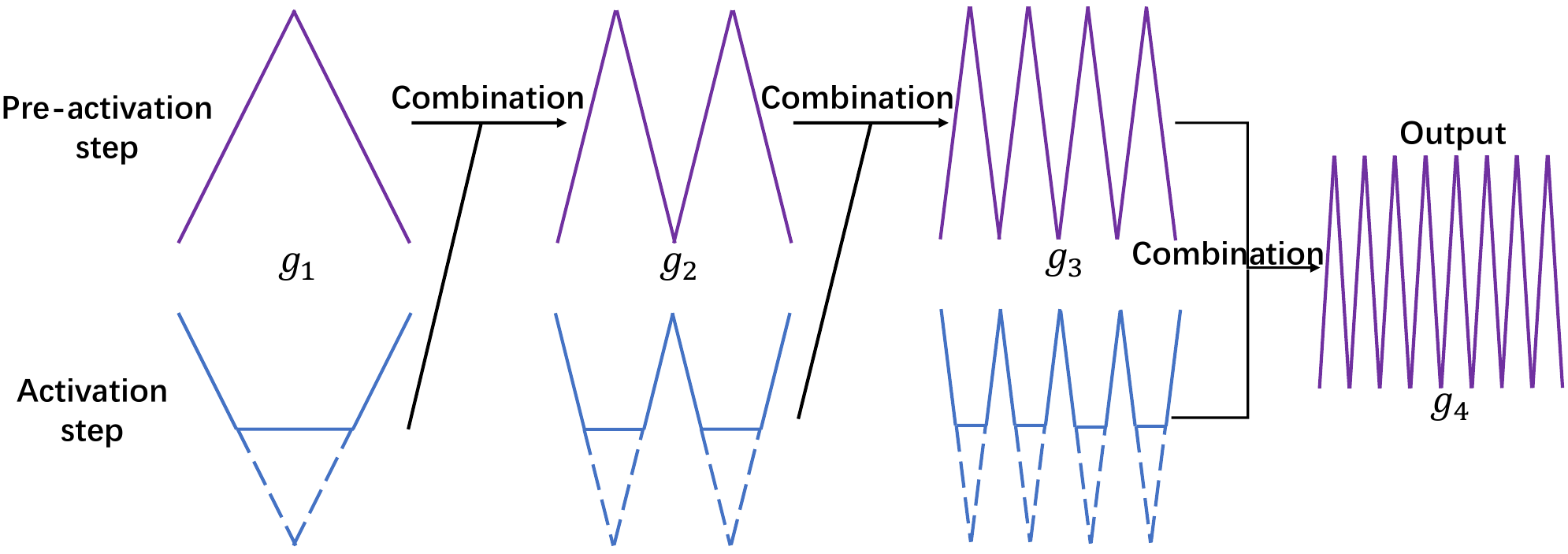}}
\caption{An illustration of constructing sawtooth functions in \cref{Sawtooth}.}
\label{Figure_Sawtooth}
\end{figure}

\begin{lemma}\label{Sawtooth}
     $g_s$ can be represented by a single-layer intra-linked ReLU network of width $s+1$.
\end{lemma}
\begin{proof}
    When $s=1$, $g_1$ is directly given by
\begin{equation*}
    g_1(x)=2x-4\sigma\left(x-0.5 \right),x\in[0,1].
\end{equation*}
Now, we use an induction step assuming that the affine combination of the first $s$ neurons represents $g_{s-1}$. Then it is sufficient to prove that the output $g_{s}$ can be given by the affine combination of the first $s$ neurons and $y=x$, since the latter can be given by only one neuron, \textit{i.e.}, $x=\sigma(x)$ when $x\in [0,1]$. Using $-g_{s}+\frac{1}{2}$ as the pre-activation of the $s$-th neuron, we have $g_{s+1}=-4\sigma\left(-g_{s}+\frac{1}{2} \right)-\frac{1}{2}g_{s}+2$, which is the affine combination of first $s+1$ neurons and $y=x$ by the induction hypothesis. \cref{Figure_Sawtooth} shows how to construct $g_4$ following the steps above. This network has a single layer with a total of $5$ neurons.
\end{proof}

\begin{remark}
    Network approximation with shortcut structures has been investigated in  \cite{fan2021sparse,lin2018resnet,Fan2023Quasi}. Although topologically adding intra-links to a layer of width $n$ can be regarded as increasing depth to $n$, they are essentially different. On one hand, intra-links need no additional activation operation, while increasing depth needs more activation. If we use activation to define a layer, intra-links do not promote depth. On the other hand, adding intra-layer links in a fully-connected network only costs the number of parameters for one layer with the same width.
\end{remark}


\textbf{Step 2: Approximation of the product function $\times(x,y)=xy:[0,1]^2\to [0,1]$}. We first use sawtooth functions to approximate the squaring function $f(x)=x^2$. Then we use the fact $xy=\frac{(x+y)^2-x^2-y^2}{2}$ to approximate $\times(x,y)$. 

Now, we illustrate how to approximate $y=x^2$ with composition and combination of sawtooth functions.

\begin{lemma}[Approximation of $y=x^2$]\label{Squaring}
Let $f_s:[0,1]\to [0,1]$ be the piecewise linear interpolation of $y=x^2$ on $x=\frac{k}{2^s},k=0,1,\ldots,2^s$. Then for any $N\in\mathbb{Z}^+$, $f_{s=N}$ can be implemented by a single-layer ReLU NN of width $N+1$. 
    
\end{lemma}

\begin{proof}
    Note that for any $j\geq 2$, $f_{j-1}-f_{j}=\frac{g_j}{2^{2j}}$, where $g_j$ is the sawtooth function in \cref{Squaring}. Thus, we have
\begin{equation*}
    f_{s}(x)=f_{1}(x)+\sum_{j=2}^{s}\left(f_{j}-f_{j-1} \right)=x-\sum_{j=1}^{s}\frac{g_j(x)}{2^{2j}}.
\end{equation*}
We consider the same hidden layer as in \cref{Squaring}. Since each $g_i$ is an affine combination of neurons in the hidden layer. Hence, $f_{s}$ can be given by a single-layer network of width $N+1$.
\end{proof}


Now using the fact that $xy=\frac{(x+y)^2-(x-y)^2}{4}$, we can easily approximate the product function as follows:

 {
\begin{lemma}[Approximation of product] \label{Multiplication}
    For any $L,N\in\mathbb{Z}^+$, there exists a function $\widetilde{\times}(x,y):[-1,1]^2\to[-1,1]$ implemented by a two-hidden-layer ReLU network width $4(N+1)$, such that for all $x,y\in [-1,1]$,
\begin{equation*}
    \left|\widetilde{\times}(x,y)-xy \right|\leq 2^{-2N}.
  \end{equation*}
Besides, $\widetilde{\times}(x,y)=0$ if $x=0$ or $y=0$.
\end{lemma}
}

\begin{proof}
 {
    Let $\tilde{f}_{N}(x)=\sigma\left(f_{N}(x)\right)+\sigma(f_{N}(-x))$, where $f_N$ is the approximation to the squaring function constructed in Lemma \ref{Squaring}. We set 
\begin{equation*}
\widetilde{\times}(x,y)=4f_{N}\left( \frac{x+y}{2}\right)-4f_{N}\left( \frac{x-y}{2}\right).
\end{equation*}
Using the fact that for any $x\in[-1,1]$, $|\tilde{f}_s(x)-x^2|\leq 2^{-2(s+1)}$, we have
\begin{equation}
    \begin{aligned}
    &|\widetilde{\times}(x,y)-xy|\\ \nonumber
    =&4\left | \tilde{f}_{N}\left( \frac{x+y}{2}\right)-\tilde{f}_{N}\left( \frac{x-y}{2}\right)-\left( \frac{x+y}{2} \right)^2+\left( \frac{x-y}{2} \right)^2 \right |\\ \nonumber
    \leq &4\left|\tilde{f}_{N}\left( \frac{x+y}{2}\right)-\left( \frac{x+y}{2} \right)^2 \right|+4\left|\tilde{f}_{N}\left( \frac{x-y}{2}\right)-\left( \frac{x-y}{2} \right)^2 \right| \nonumber\\
    \leq & 2^{-2N}.
\end{aligned}
\end{equation}
It remains to verify $\widetilde{\times}(x,y)\in[-1,1]$ for all $x,y\in[-1,1]$. Observe that $\widetilde{\times}(x,y)$ is linear in $[-1,1]\setminus D$, where 
\begin{equation*}
    D=\left\{ ( x,y)\in[-1,1]^2:x\pm y=\pm \frac{2l}{2^{N}},l=0,1,\ldots,2^{N-1} \right\}.
\end{equation*}
}

 {
Note that $\widetilde{\times}(x,y)=xy$ on $D$. Therefore, $\widetilde{\times}(x,y)$ is the piecewise linear interpolation of $xy$ on $D$. Since the values on all interpolation points are in $[-1,1]$, their linear interpolation $\widetilde{\times}(\cdot,\cdot)$ is also in $[-1,1]$.
}
 {
Since $\tilde{f}_{N}$ can be implemented by a ReLU network with width $2(N+1)$ and depth $2$. Then $\widetilde{\times}$ can be implemented by stacking two such networks in parallel. Hence the whole network is of depth $2$ and width $4(N+1)$.
}
\end{proof}
 {
We then extend the bivariate product $\widetilde{\times}(\cdot,\cdot)$ in \cref{Multiplication} for further analysis.}

 {
\begin{lemma}
\label{prod_extend}
    For any $x,y\in \mathbb{R}$, let $x^\prime,y^{\prime}$ be approximators to $x,y$ such that $|x-x^\prime|\leq \epsilon_x,|y-y^\prime|\leq \epsilon_y$. There is a function $\hat{\times}(\cdot,\cdot)$ implemented by a two-hidden-layer ReLU network of width $4(N+1)$ such that for all $x,y\in[-1,1]$,
\begin{equation*}
    \left|\hat{\times}(x^\prime,y^{\prime})-xy\right|\leq |x^\prime||y^\prime|2^{-2N}+|x^\prime|\epsilon_y+|y|\epsilon_x
\end{equation*}
\end{lemma}
}
\begin{proof}
 {
    For any $x,y\in\mathbb{R}$, we set $\hat{\times}(x^\prime,y^\prime)=|x^\prime||y^\prime|\tilde{\times}\left(\frac{x^{\prime}}{|x^\prime|},\frac{y^{\prime}}{|y^\prime|}\right)$.
    Then we have
\begin{equation*}
\begin{split}
    \left| \hat{\times}(x^\prime,y^\prime)-xy \right|&=\left| \hat{\times}(x^\prime,y^\prime)-x^{\prime}y^{\prime} \right|+\left|x^{\prime}y^{\prime}-xy\right|\\
    &\leq |x^\prime||y^\prime|2^{-2N}+\left|x^{\prime}y^{\prime}-x^{\prime}y\right|+\left|x^{\prime}y-xy\right|\\
    &\leq |x^\prime||y^\prime|2^{-2N}+|x^\prime|\epsilon_y+|y|\epsilon_x.
\end{split}
\end{equation*}
}
\end{proof}



\textbf{Step 3: Approximation of polynomials}. 
 {
\begin{lemma}[Approximation of univariate polynomials]\label{univariate}
Let $P(x)=\sum_{j=0}^{n}a_j x^j$ be an univariate polynomial of degree $n$. For any $N,L,n\in\mathbb{Z}^{+}$, $N\geq 3,n\geq 2$, there exists a function $\bar{P}(x)$ implemented by a ReLU network of depth $2\lceil \log_2^{n} \rceil$ and width $\frac{n}{2}+2n(N+1)$ such that 
\begin{equation*}
    \left\|\bar{P}(x)-P(x) \right\|_{L^{\infty}[0,1]}\leq \max_{0\leq j\leq n}|a_j|  n^2 2^{-2N-1}.
\end{equation*}
\end{lemma}
}
\begin{proof}
 {
    We first consider the approximation of $y=x^n$ for $x\in [0,1]$. This can be done by setting $h_0(x)=1$, $h_{1}(x)=x$ and performing $h_{j}(x)=\widehat{\times}\left(x,h_{j-1}(x)\right)$ for $j=2,\ldots,n$ recursively, where $\widehat{\times}$ is given by \cref{prod_extend}. Then we have for all $x\in [-1,1]$
\begin{align}
    \nonumber\left|h_{n}(x)-x^n \right|&\leq \left|\widehat{\times}\left(x,h_{n-1}(x)\right)- xh_{n-1}(x)\right|+
    \left|xh_{n-1}(x)-x^n\right|\\
    \nonumber&\leq\frac{3}{2}\cdot2^{-2N}+
    \left|h_{n-1}(x)-x^{n-1}\right|\\
    \nonumber &\leq \cdots \leq (n-1)2^{-2N}.
\end{align}
}

 {
Finally, we estimate the error. Setting $\bar{P}(x)=\sum_{j=0}^{n}a_j h_j(x)$ leads to
\begin{align}
    &\nonumber~~~~\left| \bar{P}(x)-\sum_{j=0}^{n}a_j x^j \right| \\
    &\nonumber \leq \sum_{j=0}^{n}\left| a_j \right|\left| h_j(x)-x^j \right|\\
    \nonumber &\leq \max_{0\leq j\leq n}|a_j| \sum_{j=1}^{n}\frac{3(j-1)}{2}\cdot2^{-2N}\\
    \nonumber&\leq \max_{0\leq j\leq n}|a_j| n^2 2^{-2N-1},
\end{align}
for all $x\in [-1,1]$.
}

 {
The implementation of $\bar{P}(x)$ by a ReLU network is given by \cref{Figure_univariate}. Hence, the network is of depth $2\lceil \log_2^{n} \rceil$ and width $\frac{n}{2}+2n(N+1)$.
}
\end{proof}

\begin{figure}[htb!]
\center{\includegraphics[width=\linewidth] {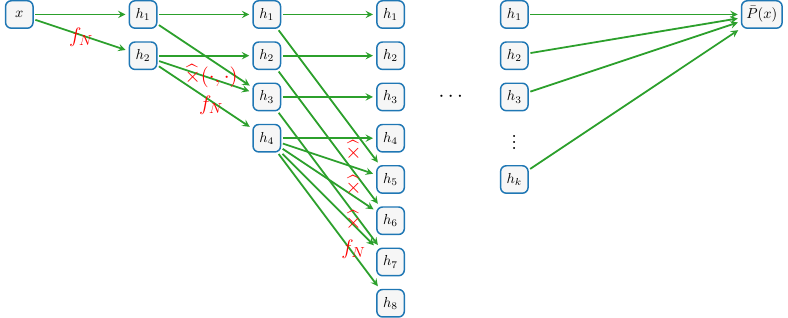}}
\caption{An illustration of the structure of a ReLU network that approximates a univariate polynomial in \cref{univariate}.}
\label{Figure_univariate}
\end{figure}

We first approximate the Chebyshev polynomial, which will be later used to approximate any Chebyshev series. We also restrict the support of this approximation within $[-1,1]$ for later construction.

 {
\begin{corollary}[Approximation of Chebyshev polynomials] \label{Chebyshev}
Given $n\in\mathbb{Z}^{+}$, for any $N \in\mathbb{Z}^+$, there exists a function $\bar{T}_{n}$ implemented by a ReLU network of depth $2(n-1)$ and width $4\left(N+1\right)+2$ or depth $2\lceil \log_2^{n} \rceil+1$ and width $\frac{n}{2}+2n(N+1)$, such that for any $\delta>0$, it approximates the Chebyshev polynomial $T_n$ with the error:
        \begin{equation*}
		\left\|T_n-\bar{T}_n\right\|_{L^{\infty}(I_\delta)}\leq n^2 3^n 2^{-2N-1}
	\end{equation*}
and $\operatorname{supp}(\Bar{T}_n)\subseteq I$.
\end{corollary}
}

\begin{proof}
 {
    The error estimate is directly from \cref{univariate}. To make the approximator be only supported in $I$, we just need to compose it with
\begin{equation*}
    h(x)=\begin{cases} x,&x\in I_\delta,\\
    0,&x\in \mathbb{R}\setminus I,\\
    \text{linear connection}, & x\in I\setminus I_\delta,
    \end{cases}
\end{equation*}
which can be directly given by a single-layer ReLU network of width $4$.
}
\end{proof}

With the above result, we are now able to conclude the proof of \cref{MainThmCube}.

\begin{proof}[Proof of \cref{MainThmCube}]
 {
	First, by \cref{Chebyshev}, for all $k\leq n$, $\bar{T}_k$ can be implemented by a ReLU network of depth $2\lceil\log_2^n\rceil$ and width $\frac{n}{2}+2n(N_1+1)$, which approximates the Chebyshev polynomial $T_k$ with the error 
	\begin{equation}\label{Approx_Chebyshev}
		\left\|T_k-\bar{T}_k\right\|_{L^{\infty}(I_\delta)}\leq n^23^n 2^{-2N_1-1},
	\end{equation}
where we set $N_1=n$.
}
 {	
	Let $\widehat{\times}$ be the product function defined in \cref{prod_extend}, which is implemented by a sub-network with width $4(N_2+1)$ and depth $2$. We consider the polynomial 
\begin{equation*}
    \mathcal{T}_{n}^{d}(f)\left(\boldsymbol{x}\right)=\sum_{j_1=0}^{n}\cdots\sum_{j_d=0}^{n}\widetilde{a}_{j_1,\ldots,j_d}T_{j_{1}}(x_1)\cdots T_{j_{d}}(x_d),
\end{equation*}
which can be rewritten $\mathcal{T}_{N}^d(f)$ as
\begin{equation*}
\begin{split}
\mathcal{T}_{n}^{d}(f)(\boldsymbol{x}) = \sum_{j_1=1}^{n} T_{j_1}(x_1) \bigg(
  & \sum_{j_2=1}^{n} T_{j_2}(x_2) \bigg(
    \sum_{j_3=1}^{n} T_{j_3}(x_3) \bigg(
      \cdots \\
      & \sum_{j_{d-1}=1}^{n} T_{j_{d-1}}(x_{d-1}) \bigg(
        \sum_{j_d=1}^{n} a_{j_1,\ldots,j_d} T_{j_d}(x_d)
      \bigg)
    \bigg)
  \bigg)
\bigg).
\end{split}
\end{equation*}
    By composing $\widehat{\times}$ and $\bar{T}_{j}$'s, $\hat{f}$ is now given by
    \begin{equation*}
\begin{split}
\hat{f}(x_1,\ldots,x_d) = \hat{\times}\bigg( & \sum_{j_1=1}^{n} T_{j_1}(x_1), \\
  & \hat{\times}\bigg( \cdots \hat{\times}\bigg( \sum_{j_{d-1}=1}^{n} T_{j_{d-1}}(x_{d-1}), \bigg( \sum_{j_d=1}^{n} a_{j_1,\ldots,j_d} T_{j_d}(x_d) \bigg) \bigg) \bigg) \bigg),
\end{split}
\end{equation*}
	where $\widetilde{a}_{j_1,\ldots,j_d}\leq (5\pi)^d\left\|f\circ \textbf{cos}\right\|_{L^{1}(Q)}$ is defined in \eqref{coefficient}. 
}

 {
    For convenience of notation, we denote
\begin{equation*}
\begin{split}
    S_{j_1,\ldots,j_{d-1}}^{(d)}(x_d)&=\sum_{j_d=1}^n a_{j_1,\ldots,j_d}T_{j_d}(x_d)\\
    S_{j_1,\ldots,j_{d-2}}^{(d-1)}(x_{d-1},x_d)&=\sum_{j_{d-1}=1}^{n}S_{j_1,\ldots,j_{d-1}}^{(d)}(x_d)T_{j_{d-1}}(x_{d}),\\
    &\cdots\\
    \mathcal{T}_{n}^{d}(f)(\boldsymbol{x})&=S^{(1)}(x_{1},\ldots,x_d)=\sum_{j_1=1}^n S_{j_1}^{(2)}(x_2,\ldots,x_d)T_{j_1}(x_{1}).
\end{split}
\end{equation*}
}

 {
Similarly, we set
\begin{equation*}
\begin{split}
    \bar{S}_{j_1,\ldots,j_{d-1}}^{(d)}(x_d)&=\sum_{j_d=1}^n a_{j_1,\ldots,j_d}\bar{T}_{j_d}(x_d)\\
    \bar{S}_{j_1,\ldots,j_{d-2}}^{(d-1)}(x_{d-1},x_d)&=\sum_{j_{d-1}=1}^{n}\hat{\times}\left(\bar{S}_{j_1,\ldots,j_{d-1}}^{(d)}(x_d),\bar{T}_{j_{d-1}}(x_{d-1})\right),\\
    &\cdots\\
    \hat{f}(\boldsymbol{x})=\bar{S}^{(1)}(x_{1},\ldots,x_d)&=\sum_{j_1=1}^n \hat{\times}\left(\bar{S}_{j_1}^{(2)}(x_2,\ldots,x_d),\bar{T}_{j_1}(x_{1})\right).
\end{split}    
\end{equation*}
}

 {
Then we compute the approximation error:
	\begin{equation}
    \label{eq_cube_err_1}
		\begin{split}
			\left\|f-\hat{f}\right\|_{L^\infty(Q_\delta)}
   &\leq \left\|f-\mathcal{T}_{n}^{d}(f)\right\|_{L^\infty(Q)}+\left\|\mathcal{T}_{n}^{d}(f)-\hat{f}\right\|_{L^\infty(Q_\delta)}\\
   &\leq 6d\omega_{f,Q}\left(\frac{1}{n}\right)+\left\|\mathcal{T}_{n}^{d}(f)-\hat{f}\right\|_{L^\infty(Q_\delta)}
		\end{split}
	\end{equation} 
}

 {
	Next we estimate the term $\left\|\mathcal{T}_{n}^{d}(f)-\hat{f}\right\|_{L^\infty(Q_\delta)}$:
\begin{equation*}
\begin{split}
    &\left\|S^{(d)_{j_1,\ldots,j_{d-1}}}-\bar{S}^{(d)}_{j_1,\ldots,j_{d-1}}\right\|_{L^{\infty}(Q_\delta)}\\
    \leq &n \max_{j_1,\ldots,j_d}a_{j_1,\ldots,j_d}\cdot n^23^n 2^{-2N_1-1}=\mathcal{O}\left(n^33^n 2^{-2N_1}\right).
\end{split}
\end{equation*}
Then 
\begin{equation*}
    \left\|\bar{S}^{(d)}_{j_1,\ldots,j_{d-1}}\right\|_{L^{\infty}(Q_\delta)}\leq \sum_{j_d=1}^n\left\| a_{j_1,\ldots,j_d}\bar{T}_{j_d}\right\|_{L^{\infty}(Q_\delta)}+\mathcal{O}\left(n^33^n 2^{-2N_1}\right)=\mathcal{O}(n),
\end{equation*}
and
\begin{equation*}
\begin{split}
    &\left\|S^{(d-1)}_{j_1,\ldots,j_{d-2}}-\bar{S}^{(d-1)}_{j_1,\ldots,j_{d-2}}\right\|_{L^{\infty}(Q_\delta)}\\
    \leq&\sum_{j_{d-1}=1}^n\left\|T_{j_{d-1}}S^{(d)}_{j_1,\ldots,j_{d-1}}-\hat{\times}\left(\bar{T}_{j_{d-1}},\bar{S}^{(d)}_{j_1,\ldots,j_{d-1}}\right) \right\|_{L^{\infty}(Q_\delta)}\\
    \leq& \sum_{j_{d-1}=1}^n \left(\left\|\bar{T}_{j_{d-1}}\right\|\left\|_{L^{\infty}(Q_\delta)} \bar{S}^{(d)}_{j_1,\cdots,j_{d-1}} \right\|_{L^{\infty}(Q_\delta)}2^{-2N_2}\right.\\
    &+\left\|S^{(d)}_{j_1,\ldots,d_{d-1}}\right\|\left\|T_{j_{d-1}}-\bar{T}_{j_{d-1}}\right\|_{L^{\infty}(Q_\delta)}\\
    &+\left.\left\|\bar{T}_{j_{d-1}}\right\|_{L^{\infty}(Q_\delta)}\left\|S^{(d)}_{j_1,\ldots,d_{d-1}}-\bar{S}^{(d)}_{j_1,\ldots,d_{d-1}}\right\|_{L^{\infty}(Q_\delta)}\right)\\
    \leq& \mathcal{O}\left(n^2 2^{-2N_2}\right)+\mathcal{O}\left(n^43^n 2^{-2N_1}\right).
\end{split}
\end{equation*}
}

 {
Using a simple induction step, we have
\begin{equation}
\label{eq_cube_err_2}
\begin{split}
    \left\|\mathcal{T}_n^d(f)-\hat{f}\right\|_{L^{\infty}(Q_\delta)}&=\left\| S^{(1)}-\bar{S}^{(1)} \right\|_{L^{\infty}(Q_\delta)}\\
    &\leq\mathcal{O}\left(n^d 2^{-2N_2}\right)+\mathcal{O}\left(n^{d+2}3^n 2^{-2N_1}\right)
\end{split}
\end{equation}
}

 {
The network to implement the procedure above is presented in \cref{Figure_KPMNet}, which is of width $\max\left\{n(d-1)\left(n/2+2n(N_1+1)\right),n(d-1)+n^{d-1}(N_2+1)\right\}$ and depth $2\log_{2}^n+2(d-1)$.
}

 {
Since each $\Bar{T}_{j}$ is supported in $I_\delta$. Hence, following the construction of $\Bar{S}^{(1)}$, we have $\text{supp}(\Bar{S}^{(1)})\subseteq I$. Finally, setting $N_1=n=N$ and $N_2=\lceil \frac{d+\alpha}{2}\log_2{N} \rceil$, we finish the proof.
}


\end{proof}

\begin{figure}[htb!]
\label{Figure_KPMNet}
\center{\includegraphics[width=\linewidth] {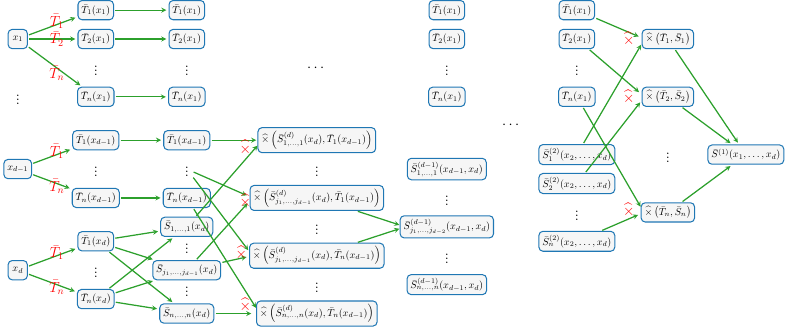}}
\caption{An illustration of constructing $\hat{f}$ by a ReLU network.}
\end{figure}

\section{Proof of \cref{Thm_Global}}\label{Sec_proof}

In this section, we prove \cref{Thm_Global}, which extends the approximation over hypercubes to the general convex polytope. In polynomial approximation theory, the fast-decreasing polynomial was well investigated \cite{TOTIK2012Nonsymmetric,Ivanov1990FastDP,Totik2015Chebyshev}, which tends to $1$ (or $\infty$) in some region and decreases fast to $0$ outside this region. The fast-decreasing polynomial plays an important role in Totik's work that extends the polynomial approximation from a hypercube to a polytope \cite{Totik2014PolynomialAO}. In this proof, we do not use fast-decreasing polynomials but adopt the top-level idea of \cite{Totik2014PolynomialAO}.

First, we give a simplified version of Lemma 4.1 in \cite{Totik2014PolynomialAO}, which provides an approach to adapt the approximation of two different regions to the union of two regions. Since we do not require the fast-decreasing function to be polynomial, our error bound is tighter.

\begin{lemma}\label{TotikTrick}
	Given any $U\subset\mathbb{R}^d$ and parallelepipeds $E,E^{\prime}$,$E\subset E^{\prime}$, let $\bar{f}_{1}$ and $\hat{f}$ be the approximations of $f$ on $U$ and $E^{\prime}$, respectively, which satisfies 
	\begin{equation*}
		\left\|\bar{f}_{1}-f\right\|_{L^{\infty}\left(U\right)}\leq e,\;\left\|\hat{f}-f\right\|_{L^{\infty}\left(E\right)}\leq e.
	\end{equation*}
		Let $\phi$ be the fast-decreasing function with respect to $E$ and $E^{\prime}$, \textit{i.e.}, 
		\begin{equation}\phi(\boldsymbol{x})
			\begin{cases}
				=1&x\in E,\\
				=0&x\notin E^{\prime},\\
				\in [0,1]&x\in E^{\prime}\setminus E.
			\end{cases}
            \label{fast-decreasing}
		\end{equation}
		Then
		\begin{equation*}
			\bar{f}_{2}=\phi \hat{f}+(1-\phi)\bar{f}_1
		\end{equation*}
		approximates $f$ with the error of
		\begin{equation}\label{TotikEq}
			\left\|\bar{f}_{2}-f\right\|_{L^{\infty}\left((U\setminus E^{\prime})\cup E\right)}\leq e.
		\end{equation}
	\end{lemma}
	
	\begin{proof}
		

  		On $E$
		\begin{equation}\label{Totik_eq2}
			\begin{split}
				\left|\bar{f}_{2}(\boldsymbol{x})-f(\boldsymbol{x})\right|&=\left|f(\boldsymbol{x})-\phi(\boldsymbol{x})\hat{f}(\boldsymbol{x})-\left(1-\phi(\boldsymbol{x})\right)\bar{f}_{1}(\boldsymbol{x})\right|\\
				&= \left|f(\boldsymbol{x})-\hat{f}(\boldsymbol{x})\right|.
			\end{split}
		\end{equation}

		
		On $U\setminus E^{\prime}$
		\begin{equation}\label{Totik_eq3}
			\begin{split}
				\left|\bar{f}_{2}(\boldsymbol{x})-f(\boldsymbol{x})\right|&=\left|f(\boldsymbol{x})-\phi(\boldsymbol{x})\hat{f}(\boldsymbol{x})-\left(1-\phi(\boldsymbol{x})\right)\bar{f}_{1}(\boldsymbol{x})\right|\\
				&= \left|f(\boldsymbol{x})-\bar{f}_{1}(\boldsymbol{x})\right|.
			\end{split}
		\end{equation}
		
		We conclude the proof after combining \eqref{Totik_eq2} and \eqref{Totik_eq3}.
	\end{proof}

Next, we construct ReLU networks to implement a fast-decreasing function concerning two parallelepipeds.

\begin{lemma}\label{FastDecreasing}
    Let $E\subset\mathbb{R}^d$ be a parallelepiped. Then for any $0<\lambda<1$, there is a fast-decreasing function $\phi$ with respect to $E$ and $E^{\lambda}$ implemented by a ReLU network with depth $d$ and width $4d$.
\end{lemma}

\begin{proof}
   Let $\mathcal{A}:\mathbb{R}^{d}\to \mathbb{R}^{d}$ be the affine transform mapping $E$ to $Q^d$. Note that there exists some $\delta>0$ such that $\mathcal{A}(E^{\lambda})=Q_\delta$. Let $\varphi$ be the fast-decreasing function with respect to $I$ and $I_\delta$ in dimension 1 such as Eq. \eqref{fast-decreasing}, which can be easily implemented by a single layer ReLU network with width 4. We use $\mathcal{A}_j \boldsymbol{x}$ to denote the $j$-th element of $\mathcal{A}\boldsymbol{x}$. Then $\phi$ is given by
\begin{equation} \phi(\boldsymbol{x})=\widetilde{\gamma}_{d}\left(\varphi(\mathcal{A}_{1}\boldsymbol{x}),\ldots,\varphi(\mathcal{A}_{d}\boldsymbol{x})\right),
\end{equation}
where $\widetilde{\gamma}_d$ is a $d$-dimensional function which we recursively define as
\begin{equation*}
\begin{split}
    \tilde{\gamma}_1(x_1)&=x_1,\\    \tilde{\gamma}_d(x_1,\ldots,x_d)&=\tilde{\gamma}\left(\tilde{\gamma}_{d-1}(x_1,\ldots,x_{d-1}),x_d\right),
\end{split}
\end{equation*}
and we need $\widetilde{\gamma}(x,y)$ to satisfy $\widetilde{\gamma}(0,y)=\widetilde{\gamma}(x,0)=0$ and $\widetilde{\gamma}(1,1)=1$.

Hence the most simple choice of $\widetilde{\gamma}$ is given by a single-layer ReLU network with width 3:
 \begin{equation}
     \widetilde{\gamma}(x,y)=2\sigma\left(\frac{x+y}{2}\right)-2\sigma\left(\frac{x}{2}\right)-2\sigma\left(\frac{y}{2}\right),
 \end{equation}
Then $\widetilde{\gamma}_d$ can be implemented by a ReLU network with depth $d-1$ and width $d+1$.
 The entire network to implement $\phi$ needs the first layer with width $4d$ to denote the affine transform, followed by a sub-network implementing $\tilde{\gamma}$. Hence, the entire ReLU network takes width $4d$ and depth $d$ to express $\phi(\boldsymbol{x})$.
 \end{proof}

Now, we prove \cref{Thm_Global} based on the above lemmas. The key idea is to replace the multiplication in \cref{TotikTrick} with a bivariate function $\widecheck{\times}(\cdot,\cdot)$ and use the fast-decreasing function in \cref{FastDecreasing}.

\begin{proof}[Proof of \cref{Thm_Global}]
Let $\phi_j$ be the fast-descreasing function with respect to $K_j$ and $K_j^{\lambda}$ for some small $\lambda$ and $j=1,\ldots,k$. We define $\bar{f}_j$ in a similar way as \cref{TotikTrick}:
    \begin{equation}\label{bar_f}
    \begin{split}
        \bar{f}_{1}&=\hat{f}_1,\\
        \bar{f}_j&=\phi_j \hat{f}_j+\left(1-\phi_j\right)\bar{f}_{j-1},j=2,\ldots,k.
    \end{split}
    \end{equation}
    Then from \cref{TotikTrick}, we immediately obtain
    \begin{equation*}
        \left\| f-\bar{f}_k \right\|_{L^{\infty}\left(\cup_{j=1}^{k}K_j\setminus \cup_{j=1}^{k}\left( K_j\setminus K_j^{\lambda} \right) \right)}\leq err.
    \end{equation*}

To approximate $\bar{f}_k$ with a neural network, we replace the multiplication in \eqref{bar_f} with a bivariate product function $\widecheck{\times}(\cdot,\cdot):[-M,M]\times [0,1]\to [-M,M]$ for some $M>0$, which can be implemented by a ReLU network with depth $2$ and width $4(N+1)$ for $N \in \mathbb{Z}^+$. 

Now, we define
    \begin{equation}\label{Def_fTilde}
    \tilde{f}_1=\hat{f}_1,\; \tilde{f}_j=\widecheck{\times}\left(\hat{f}_{j}, \phi\right)+\widecheck{\times}\left(\tilde{f}_{j-1},1-\phi \right),j=2,\ldots,k,
    \end{equation}
Then for any $\boldsymbol{x}\in \cup_{j=1}^{k}K_j\setminus \cup_{j=1}^{k}\left( K_j\setminus K_j^{\lambda} \right)$, we have 
\begin{equation}\label{fTildeErr}
\begin{split}
    \left| \widetilde{f}_{k}-f \right|\leq & \left| \widetilde{f}_{k}-\bar{f}_k \right|+ \left| f-\bar{f}_k \right|  \\
    \leq & \left| \widecheck{\times}\left( \widetilde{f}_{k-1},1-\phi_k \right)-\bar{f}_{k-1}(1-\phi_k) \right|+\left| \widecheck{\times}\left(\hat{f}_{k},\phi_k \right)-\hat{f}_k\phi_k  \right|+ err \\
    \leq & \left| \widecheck{\times}\left( \widetilde{f}_{k-1},1-\phi_k \right)- \widetilde{f}_{k-1}(1-\phi_k)\right|\\
    &+\left| (1-\phi_k)(\widetilde{f}_{k-1}-\bar{f}_{k-1}) \right|+\left| \widecheck{\times}\left(\hat{f}_{k},\phi_k \right)-\hat{f}_k\phi_k  \right|+ err.\\
    \leq & \left| \widetilde{f}_{k-1} \right|2^{-2N+1}+\left| \widetilde{f}_{k-1}-\bar{f}_{k-1} \right|+\left| \hat{f}_{k} \right|2^{-2N+1}+ err \\
    \leq \cdots &  \leq \cdot 2^{-2N+1}\sum_{j=1}^{k-1}\left(\left|\hat{f}_{j+1}\right|+\left|\tilde{f}_{j}\right|\right)+ err.
\end{split}
\end{equation}
where in the fourth inequality, we leverage 
\begin{equation}
\begin{split}
    \left|\widecheck{\times}(g,\phi)-g\phi\right|\leq 2|g|\cdot 2^{-2N},\\
    \left|\widecheck{\times}(g,1-\phi)-g(1-\phi)\right|\leq 2|g|\cdot 2^{-2N},
\end{split}
\end{equation}
for any function $g$ and $\phi\in [0,1]$, and the last inequality in \eqref{fTildeErr} follows after repeating the previous steps.

Next we estimate $\left|\hat{f}_{j}\right|$ and $\left|\tilde{f}_{j}\right|$. On one hand, for any $\boldsymbol{x}\in\mathbb{R}^d$, we have by assumption
\begin{equation}\label{fHat}
    \left|\hat{f}_{j}\right|\leq M.
\end{equation}
On the other hand, by the definition \eqref{Def_fTilde} of $\tilde{f}_{j},j\geq 2$, we have for any $\boldsymbol{x}\in\mathbb{R}^d$
\begin{equation}\label{fTilde}
\begin{split}
    \left| \tilde{f}_{j} \right|&\leq \left| \tilde{f}_{j-1} \right|+\left| \hat{f}_{j} \right|\leq \cdots\\
    &\leq \left| \hat{f}_j \right|+\left| \hat{f}_{j-1} \right|+\cdots \left| \hat{f}_1 \right|\\
    &\leq jM.
\end{split}
\end{equation}
Hence, combining \eqref{fTildeErr}, \eqref{fHat}, and \eqref{fTilde}, for any $\boldsymbol{x}\in\cup_{j=1}^{k}K_j\setminus \cup_{j=1}^{k}\left( K_j\setminus K_j^{\lambda} \right)$,
\begin{equation}
\begin{split}
    \left| \widetilde{f}_{k}-f \right|\leq &  2^{-2N+1}\sum_{j=1}^k (j+1)M+err\\
   \leq & (k+2)^2M2^{-2N}+err.
\end{split}
\end{equation}

Finally, we estimate the size of the network that implements $\tilde{f}_k$, which is illustrated in \cref{Figure_ApproxTotik}. First, $k$ sub-networks of depth $D$ and width $W$ are stacked in parallel for $\hat{f}_1,\hat{f}_2,\ldots,\hat{f}_k$, respectively. Simultaneously, $k-1$ sub-networks of depth $d$ and width $4d$ gives $\phi_2,\ldots,\phi_k$. Hence this part is of depth $\max\left\{W,d\right\}$ and width $kW+4(k-1)d$. Next, in each step we need two sub-networks to implement two $\widecheck{\times}$ and at most $2(k-2)$ identity mappings. Therefore, it needs width $8(N+1)+2(k-2)$ and depth $2$. Following $k-1$ such steps, $\tilde{f}_k$ is finally given by a ReLU network of width $\max\left\{8(N+1)+2(k-2),kW+4(k-1)d\right\}$ and depth $\max\{D,d\}+2(k-1)$.
\end{proof}

\begin{figure}[htb!]
\center{\includegraphics[width=\linewidth] {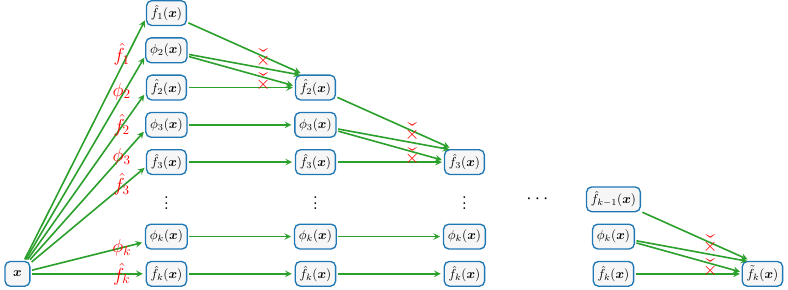}}
\caption{An illustration of constructing $\tilde{f}_k$ by ReLU networks.}
\label{Figure_ApproxTotik}
\end{figure}

\section{Extension to Other Function Classes}
 {
In this part, we extend our main result to other functions such as smooth and analytic functions. This shows that our construction is task-aware and can be highly adaptive to the target function, \textit{i.e.}, the network topology varies depending on the local regularity.
}
\begin{proof}[Proof of \cref{thm:MainThmAnalytic}]
 {
        It suffices to consider the case $K=Q=[-1,1]^d$, since the rest follows immediately from \cref{Thm_Global}. As shown in Theorem 4.2 in \cite{trefethen2017multivariate}, there exists a Chebyshev series $P_n$:
\begin{equation*}
P_n(\boldsymbol{x})=\sum_{j_1=0}^n\cdots\sum_{j_d=1}^na_{j_1,\ldots,j_d}T_{j_1}(x_1)\cdots T_{j_d}(x_d),
\end{equation*} 
such that
\begin{equation*}
    \left\|f-P_n\right\|_{L^{\infty}(Q)}\leq \mathcal{O}_{\epsilon}\left((h+\sqrt{1+h^2})^{-n/\sqrt{d}}\right)
\end{equation*}
and
\begin{equation*}
    a_{j_1,\ldots,j_d}=\mathcal{O}\left(\rho^{-\sqrt{j_1^2+\cdots+j_d^2}}\right)
\end{equation*}
as $j_1+\cdots+j_d\to\infty$, which is bounded. Hence, we can use the same strategy as in proving \cref{thm:MainThmContinuous}. That is, we have, followed from \cref{eq_cube_err_1,eq_cube_err_2}, 
\begin{equation*}
		\begin{split}
			\left\|f-\hat{f}\right\|_{L^\infty(Q_\delta)}
   &\leq \left\|f-P_n\right\|_{L^\infty(Q)}+\left\|P_n-\hat{f}\right\|_{L^\infty(Q_\delta)}\\
   &\leq \mathcal{O}_{\epsilon}\left((h+\sqrt{1+h^2})^{-n/\sqrt{d}}\right)+\mathcal{O}\left(n^d 2^{-2N^{\prime}}\right)+\mathcal{O}\left(n^{d+2}3^n 2^{-2N}\right),
		\end{split}
\end{equation*}
where $\hat{f}$ is given by a ReLU network of the same structure as in \cref{MainThmCube}, \textit{i.e.}, a ReLU network of width $\max\left\{n(d-1)\left(n/2+2n(N_1+1)\right),n(d-1)+n^{d-1}(N_2+1)\right\}$ and depth $2\log_{2}^n+2(d-1)$. Then the result follows immediately after setting $N_1=\frac{n}{\sqrt{d}}\log_{3/4}\rho$ and $N_2=\frac{n}{\sqrt{d}}\log_{4}\rho$.
}
\end{proof}

\section{Conclusions and future directions}\label{sec_discuss}

Our main result is a network construction that approximates an arbitrary piecewise continuous function by partitioning the input space into polytopes. In practice, a ReLU network indeed partitions the input space into polytopes, which means that from this angle, our construction is more realistic than universal approximation based on hypercubes. We do not claim that our construction is completely realistic. To bridge the gap, one future direction should be connecting the construction with learnability, \textit{i.e.}, is there any network construction that is learnable by gradient descent?

Notably, the techniques of the Jackson-type kernel can also be applied to other types of functions such as smooth functions and analytical functions that can be approximated by the Chebyshev series. Both polytopes and hypercubes can be addressed. We think the Jackson-type approximation can technically greatly enrich the deep learning approximation theory in terms of going beyond the Taylor expansion, which should facilitate more research opportunities.

Another interesting direction is the manifold assumption that real-world data sets often lie in low-dimensional structures. Accordingly, one can use the manifold assumption to update our main theorem with the intrinsic dimensionality, where the manifold is constructed by triangulation and applicable to our approximation scheme.

\section{Acknowledgement}\label{acknowled}

The authors are grateful for Dr. Huan Xiong (Harbin Institute of Technology) and Dr. Han Feng (City University of Hong Kong)’s suggestions in finishing this work.  This work was supported in part by the National Key R$\&$D Program of
 China under Grants 2021YFE0203700, NSFC/RGC N$\_$CUHK 415/19, ITF
 MHP/038/20, CRF 8730063, RGC 14300219, 14302920, and 14301121; in
 part by CUHK Direct Grant for Research.

\bibliographystyle{siamplain}
\bibliography{references}
\end{document}